\documentclass[11pt]{article}

\usepackage{styles}
\bluehyperref
\usepackage[numbers]{natbib}
\usemedgeometry

\usepackage{overpic}
\usepackage{enumitem}

\usepackage{macros}

\begin{document}

\title{Finding a stationary point of a stochastic convex problem}
\author{Felipe Areces ~~~~ John Duchi ~~~~ Malo Sommers \\ Stanford University}
\date{July 2026}

\maketitle


\begin{abstract}
  We consider the problem of finding stationary points for
  stochastic convex optimization problems.
  Rather than surrogates to stationarity, such as a
  proximity-to-stationarity guarantee or small gradient of the Moreau
  envelope, we ask for a stronger notion: that the subdifferential of the
  objective actually contains a small element.
  This criterion is non-trivial, because subdifferentials of convex
  functions fail to converge uniformly, even in arbitrarily small
  neighborhoods of the optimum.
  Our convergence guarantees rely on dimension theory to decompose the graph
  of the subdifferential of a convex function, showing how
  stochastic sampling preserves ``pieces'' of
  these graphs, and allowing
  effective application of proximal-point-like methods.
\end{abstract}

\vspace{.5cm}


\vspace{.5cm}


\newpage


\section{Introduction}

Consider a stochastic convex optimization problem,
where for a collection $\{\loss_z\}_{z \in \mc{Z}}$
of convex functions, we wish to minimize
the population objective
\begin{align}
  \label{eqn:objective}
  \poploss_P(x) \defeq \E_P[\loss_Z(x)] = \int \loss_z(x) dP(z).
\end{align}
Upon being asked what it means to solve such a problem, an optimizer might
say ``finding a point near the minimizer'' or ``with approximately optimal
function value''~\cite{HiriartUrrutyLe93,Nesterov04,BoydVa04,Bertsekas15}.
Yet these notions of optimality can become unsatisfactory:
consider finding a median of a distribution $P$, i.e.,
$x \in \R$ such that
\begin{align*}
  \half \le P(Z \le x)
  ~~ \mbox{and} ~~
  \half \le P(Z \ge x)
  ~~~ \mbox{for}~ Z \sim P.
\end{align*}
%
Finding a median is equivalent to minimizing the expected absolute error
$\loss_z(x) = |x - z|$.

Two examples highlight issues with these optimality criteria
here.
First, even if the median $x\opt$ is unique, seeking a point near $x\opt$
may be a fool's errand.

\begin{example}
  \label{example:no-median}
  Let $0 < \epsilon \le 1$ and
  the distributions $P_0$ and $P_1$ on
  $Z \in \{0, 1\}$ satisfy
  \begin{equation*}
    P_0(Z = 0) = P_1(Z = 1) = \frac{1 + \epsilon}{2}
    ~~~ \mbox{and} ~~
    P_0(Z = 1) = P_1(Z = 0) = \frac{1 - \epsilon}{2}.
  \end{equation*}
  These have unique medians $x\opt_0 = 0$
  and $x\opt_1 = 1$, while
  their KL-divergence $\dkl{P_0}{P_1} =
  \epsilon \log \frac{1 + \epsilon}{1 - \epsilon} \le 2 \epsilon^2
  + \epsilon^4$ for $0 \le \epsilon \le \frac{1}{4}$.
  A standard two-point lower bound~\cite[Ch.~15.2]{Wainwright19} shows that
  any estimator $\what{x}_n$ based on a sample $Z_1, \ldots, Z_n$ drawn
  i.i.d.\ from $P_0$ or $P_1$ must satisfy
  \begin{align*}
    P_0\left(\left|\what{x}_n - x\opt_0\right| \ge \half\right)
    + P_1\left(\left|\what{x}_n - x\opt_1\right| \ge \half \right)
    \ge 1 - \sqrt{\frac{n}{2} \dkl{P_0}{P_1}}
    \ge 1 - \sqrt{n \epsilon^2 \left(1 + \frac{\epsilon^2}{2}\right)}.
  \end{align*}
  Until $n \gg \frac{1}{\epsilon^2}$, no method can reliably estimate the
  median to even constant accuracy.
\end{example}

\noindent
Faced with this difficulty, our hypothetical optimizer instead seeks small
excess objective:

\begin{example}
  \label{example:no-scale}
  Take $\loss_z(x) = |x - z|$, let $\poploss_P(x) =
  \E_P[\loss_Z(x)]$, and fix a desired
  accuracy $\epsilon > 0$,
  where $\what{x}$ minimizes $\poploss_P$ to accuracy $\epsilon$ if
  $\poploss_P(\what{x}) - \inf_x \poploss_P(x) \le \epsilon$.
  For any distribution $P$ supported on $[-\epsilon/10, \epsilon/10]$, the
  point $\what{x} = \epsilon / 2$ certainly satisfies $\poploss_P(\what{x})
  - \inf_x \poploss_P(x) \le \epsilon$, but
  \begin{align*}
    P(Z \le \what{x}) = 1
    ~~ \mbox{and} ~~
    P(Z \ge \what{x}) = 0,
  \end{align*}
  so that $\what{x}$ fails quite spectacularly to behave like a median.
\end{example}

These examples show that proximity to a minimizer and small excess
risk fail to capture the distributional properties defining a median.
In this case, the relevant criterion becomes
stationarity for the absolute loss:
for $\loss_z(x) = |x - z|$ we have
$\partial \loss_z(x) = \{\sign(x - z)\}$, where $\sign(0) = [-1,1]$,
giving population subdifferential
\begin{equation*}
  \partial \poploss_P(x) = P(x > Z) - P(x < Z) + [-1, 1] \cdot P(Z = x).
\end{equation*}
Therefore, for any $x \in \R$,
the subdifferential
$\partial \poploss_P(x)$ contains $g$ with $|g| \le 2\epsilon$
if and only if
\begin{align*}
  P(Z \ge x) \ge \half - \epsilon
  ~~ \mbox{and} ~~
  P(Z \le x) \ge \half - \epsilon,
\end{align*}
addressing the issues Examples~\ref{example:no-median}
and~\ref{example:no-scale} highlight.

More generally, we consider finding near stationary points
of convex losses.
Let $X \subset \R^d$ be closed convex and $\{\loss_z\}_{z \in \mc{Z}}$ be a
collection of convex functions $\loss_z : \R^d \to \R \cup \{+\infty\}$,
where $\relint \dom \loss_z \supset X$, so that each is subdifferentiable
over $X$.
For a distribution $P$ on $Z \in \mc{Z}$, recall~\eqref{eqn:objective} the
population loss $\poploss_P(x) = \E_P[\loss_Z(x)]$.
Then the standard first-order conditions for
optimality~\cite[Thm.~VII.1.1.1]{HiriartUrrutyLe93} state that $x\opt$
minimizes $\poploss_P(x)$ over $x \in X$ if and only if
\begin{equation*}
  0 \in \partial
  \poploss_P(x\opt) + \normalcone_X(x\opt)
  ~~~ \mbox{where} ~~~
  \normalcone_X(x) = \{v
  \mid \<v, y - x\> \le 0 ~ \mbox{for~all~} y \in X\}
\end{equation*}
denotes the normal cone to $X$ at $x$.
We therefore seek estimators achieving \emph{$\epsilon$-stationarity},
that is, $\what{x}$ for which
\begin{equation}
  \label{eqn:stationary-point-goal}
  \dist\left(0, \partial \poploss_P(\what{x}) + \normalcone_X(\what{x})\right)
  \le \epsilon.
\end{equation}
We develop estimators $\what{x}_n$ based on a sample $Z_1, \ldots, Z_n
\simiid P$ satisfying
\begin{equation*}
  \dist\left(0, \partial \poploss_P(\what{x}_n) + \normalcone_X(\what{x}_n)
  \right) \to 0
  ~~ \mbox{in probability~as~} n \to \infty.
\end{equation*}

\subsection{Motivation}

The problem of finding $\epsilon$-stationary
points~\eqref{eqn:stationary-point-goal}, and the more abstract variants we
consider later of finding near zeros of maximal monotone operators and
variational inequalities, arise in many areas of statistics and
optimization.
For us, a main motivation comes from a direction within
statistics and machine learning that seeks ``distribution-free'' predictive
inference~\cite{VovkGaSa99, Lei14, LeiWa14, RomanoPaCa19,
  CauchoisGuDu21, BarberCaRaTi21a, AngelopoulosBa23}: algorithms for
prediction problems that output prediction sets in which
true outcomes are guaranteed to lie with a prescribed probability,
without making any assumptions on the underlying data generating
distribution.
The somewhat extended examples here show how these guarantees are equivalent
to obtaining the approximate stationary
guarantee~\eqref{eqn:stationary-point-goal}.

\begin{example}[Conformal prediction in regression]
  Consider predictive
  inference in a regression problem of predicting an outcome $Y \in \R$ using
  an input $X \in \mc{X}$.
  In these problems, one is given a predictor $h : \mc{X} \to \R$ and a
  sample $(X_i, Y_i)_{i=1}^n \simiid P$.
  One then seeks a confidence set $\what{C}_n : \mc{X} \toto \R$
  that, on a new example $X_{n + 1} \sim P$, satisfies
  \begin{align*}
    \P\left(Y_{n + 1} \in \what{C}_{n}(X_{n+1})\right) \ge 1 - \alpha
  \end{align*}
  for a level $\alpha$, that is, the confidence set contains
  the true outcome $Y_{n + 1}$ with probability at least $1 - \alpha$.
  To write this as finding stationary points in a convex loss minimization
  problem, take the standard choice~\cite{LeiWa14,BarberCaRaTi21a,
    AngelopoulosBa23} of an interval of $\pm \tau$ around $h(x)$,
  \begin{equation*}
    C_\tau(x) \defeq \left\{y \in \R \mid |y - h(x)| \le \tau\right\}
    = \left[h(x) - \tau, h(x) + \tau\right].
  \end{equation*}
  Written as loss minimization, taking the quantile loss
  $\loss_z(\tau) = \alpha \hinge{\tau - z} + (1 - \alpha) \hinge{z - \tau}$ on
  the errors $z = |h(x) - y|$ and population
  loss $\poploss_P(\tau) = \E_P[\loss_Z(\tau)]$,
  the corresponding subdifferential is
  \begin{align*}
    \partial \poploss_P(\tau)
    = \alpha P(\tau > Z)
    - (1 - \alpha) P(\tau < Z)
    + [-(1 - \alpha), \alpha] \cdot P(\tau = Z),
  \end{align*}
  so $\tau\opt \in \argmin_\tau \E_P[\loss_Z(\tau)]$ if and only if $1 -
  \alpha \le P(Z \le \tau\opt)$ and $P(Z < \tau\opt) \le 1 - \alpha$.
  The near-stationarity condition $\dist(0, \partial \poploss_P(\tau)) \le
  \epsilon$ implies $1 - \alpha - \epsilon \le P(\tau \ge Z)$, that is,
  any $\tau$ satisfying~\eqref{eqn:stationary-point-goal}
  guarantees the predictive coverage
  \begin{align*}
    \P\left(Y \in C_\tau(X)\right)
    = \P(\left|Y - h(X)\right| \le \tau) \ge 1 - \alpha - \epsilon.
  \end{align*}
  So $\epsilon$-stationarity~\eqref{eqn:stationary-point-goal}
  implies valid confidence sets for predictions.
\end{example}

\begin{example}[Conditional prediction sets]
  To obtain approximately
  ``conditional on $X$'' guarantees~\cite{GibbsChCa25,
    Duchi25}, for a feature mapping $\phi(x) \in \R^d$, one may consider
  confidence sets of the form $C_\theta(x) = \{y : |y - h(x)|
  \le \theta^\top \phi(x)\}$.
  %
  One exemplar consequence of this follows:
  let $\mc{W}$ consist of the nonnegative
  functions $w(x) = v^\top \phi(x)$, $v \in \R^d$, and define the weighted
  distribution $\P_w(A) = \E[w(X) \indic{X \in A}] / \E[w(X)]$.
  Then taking $z$ to be errors $z = |h(x) - y|$ and the quantile loss
  $\loss_z(\tau) = \alpha \hinge{\tau - z} + (1 - \alpha) \hinge{z - \tau}$,
  \citet{GibbsChCa25} observe
  \begin{observation}
    \label{observation:conditional-coverage}
    Let $\poploss_P(\theta) = \E_P[\loss_Z(\phi(X)^\top\theta)]$.
    If $\dist(0, \partial \poploss_P(\theta)) \le \epsilon$, then
    \begin{align*}
      \P_w\left(Y \in C_\theta(X)\right)
      \ge 1 - \alpha - \epsilon  \cdot \frac{\ltwo{v}} {\E[w(X)]}
    \end{align*}
    for all nonnegative weight functions $w(x) = v^\top \phi(x)$.
  \end{observation}
  \noindent
  (Appendix~\ref{sec:proof-conditional-coverage} includes a proof for
  completeness.)
  Once again, $\epsilon$-stationary points~\eqref{eqn:stationary-point-goal}
  imply predictive validity.
\end{example}


With these examples as motivation, we see that we truly wish to obtain
(approximately) stationary points $\what{x}$: not points that approximately
minimize $\poploss_P$, or are close to a minimizer, but truly
satisfy $\dist(0, \partial \poploss_P(\what{x})+\normalcone_X(\what{x})) \to 0$.
We thus return to the general problem of finding an $\epsilon$-stationary
point of a convex objective given a sample $Z_1, \ldots, Z_n \simiid P$.


\subsection{Problem framing and background}
\label{sec:related-work}


\subsubsection{Monotone operator formulation}

The stationarity target~\eqref{eqn:stationary-point-goal} fits within
the more general setting of monotone operators and variational
inequalities~\cite{RockafellarWe98, BauschkeCo17}.
A set-valued operator $A:\R^d \toto \R^d$ is monotone on
$X \subset \R^d$ if every $x, y \in X$, $u \in A(x)$, and $v\in A(y)$
satisfies $\< u-v, x-y\> \geq 0$, with
$\dom A = \{x \mid A(x) \neq \emptyset\}$ and graph
$\gr A = \{(x, u) \mid u \in A(x)\}$.
It is maximally monotone if no other monotone operator has graph
strictly containing $\gr A$.
When $F$ is a proper closed convex function, its subdifferential
$\partial F$ is maximal monotone~\cite[Theorem
12.17]{RockafellarWe98}.
%
%
The more general framing for
problem~\eqref{eqn:stationary-point-goal}, and that in which we state
our results, thus follows: let $X \subset \R^d$ be a closed convex
set.
Then for a maximal monotone operator $A$, we seek $x\opt \in X$ for
which there exists $u \in A(x\opt)$ such that $\<u, x - x\opt\> \ge 0$
for all $x \in X$.
For the normal cone
$\normalcone_X(x) = \{v \mid \<v, y - x\> \le 0, ~ \mbox{all~}y \in X\}$,
we equivalently~\cite[Ex.~6.13]{RockafellarWe98} seek $x\opt$
satisfying
\begin{align*}
  0 \in A(x\opt) + \normalcone_X(x\opt).
\end{align*}
Within deterministic optimization, Rockafellar~\citep{Rockafellar76}
shows that for maximal monotone operators, the proximal point
algorithm in quite some generality generates iterates $x_k$ satisfying
$\dist(0, A(x_k) + \normalcone_X(x_k)) \to 0$.

The statistical setting appears to introduce substantial difficulty.
We assume we have a collection of maximal monotone mappings
$A_z(\cdot) : \R^d \toto \R^d$ for $z \in \mc{Z}$.
We will not overly concern ourselves with measurability, but follow
standard practice~\cite[e.g.][]{Bianchi16} to assure that integrals
have appropriate definitions.
The standard Aumann integral~\cite[Ch.~8]{AubinFr90} defines
\begin{equation*}
  A_P(x) \defeq \int A_z(x) dP(z)
  = \left\{\int a_z(x) dP(z) \mid a_z(x)
  ~ \mbox{is~a~measurable~selection~of}~ A_z(x)\right\},
\end{equation*}
meaning that $z \mapsto a_z(x)$ is measurable and $a_z(x) \in A_z(x)$
for $P$-almost-all $z$.
We assume $A_P$ is maximal monotone over $X$, that is,
$A_P + \normalcone_X$ is maximal monotone.
(When $A_z(x) = \partial \loss_z(x)$ is the subdifferential of a
convex function, these conditions hold essentially
automatically~\cite{Bertsekas73}.)

\subsubsection{The failure of uniform subdifferential convergence}

Of course, we only have sample access to $P$ via $Z_i \simiid P$.
Letting $P_n = \frac{1}{n} \sum_{i = 1}^n \pointmass_{Z_i}$ denote the
empirical distribution and
\begin{align*}
  \dhaus(X, Y) = \max\Big\{\sup_{x \in X} \inf_{y \in Y}
  \ltwo{x - y}, \sup_{y \in Y} \inf_{x \in X} \ltwo{x - y}\Big\}
\end{align*}
denote the Hausdorff distance between sets $X$ and $Y$, the most
natural notion~\cite{ShapiroXu07, TianRo25} of convergence of
$A_{P_n}$ to $A_P$ would be that
\begin{equation}
  \label{eqn:uniform-convergence-subdifferentials}
  \sup_{x\in X}\dhaus(\popop(x),\empop(x)) \to 0
\end{equation}
(in probability, expectation, or almost surely).
If we had such convergence, the empirical solution $\what{x}_n \in X$
satisfying $0 \in \empop(\what{x}_n) + \normalcone_X(\what{x}_n)$
would satisfy
$\dist(0, \popop(\what{x}_n) + \normalcone_X(\what{x}_n)) \to 0$.
(Indeed, let $w_n \in \normalcone_X(\what{x}_n)$ and
$a_n \in \empop(\what{x}_n)$ satisfy $0 = a_n + w_n$.
Then there is $a_n\opt \in \popop(\what{x}_n)$ satisfying
$\norms{a_n\opt - a_n} \to 0$, and so
$a_n\opt + w_n = a_n + w_n + o(1) \to 0$.)

Unfortunately, the
convergence~\eqref{eqn:uniform-convergence-subdifferentials} cannot
generally hold.
In all dimensions $d \ge 2$, even subdifferentials of convex functions
$\loss_z : \R^d \to \R$ exhibit quite spectacular failures of
convergence~\cite{TianRo25}.
Indeed, \citet{TianRo25} exhibit a collection of $1$-Lipschitz convex
functions $\loss_z : \R^2 \to \R$ and points $x_n$, which are
functions of the sample $P_n$, satisfying that
\begin{align*}
  \dhaus(\partial \poploss_{P_n}(x_n), \partial \poploss_P(x_n))
  \ge \half
  ~~~ \mbox{with~probability~1}
\end{align*}
for all sample sizes $n$.
In Section~\ref{sec:baby-challenges}, we give a simple example,
motivated by theirs, for which an empirical minimizer satisfies the
same failure.

\subsubsection{Existing stationarity guarantees}

These difficulties inspire several alternative approaches.
One restricts the class of functions under consideration.
For example, \citet{ShapiroXu07} show that if the population objective
$\poploss_P$ is differentiable (which, by convexity, implies
continuous differentiability), then uniform
convergence~\eqref{eqn:uniform-convergence-subdifferentials} holds.
\citet{Ruan25} explores a related approach, showing that classes of
functions whose subdifferentials have finite Vapnik-Chervonenkis
dimension enjoy uniform convergence, even in weakly convex settings.
In this paper, we make no structural assumptions on the function class
$\{\loss_z\}_{z \in \mc{Z}}$, and so cannot achieve uniform
convergence~\eqref{eqn:uniform-convergence-subdifferentials}, but show
how a randomized proximal-point sampling method achieves
stationarity~\eqref{eqn:stationary-point-goal}.

The other main approach focuses on weaker guarantees.
These rely on one of several essentially equivalent conditions, which
roughly show that an algorithm generates iterates near approximately
stationary points.
This viewpoint follows classically from Ekeland's variational
principle~\citep{Ekeland74} and the Br{\o}ndsted-Rockafellar
theorem~\citep{BrondstedRo65}; in keeping with our development, we
describe the approach via proximal operators.

For $\lambda > 0$ and a convex function $F$ or, respectively, maximal
monotone operator $A$, let
\begin{align*}
  \proxmap_{\lambda F}(x) \defeq \argmin_y \left\{\lambda F(y)
  + \half \ltwo{y - x}^2
  \right\}
  ~~ \mbox{and} ~~
  \proxmap_{\lambda A}(x) \defeq (I + \lambda A)^{-1}(x).
\end{align*}
%
Standard calculations show that if
$x_\lambda = (I + \lambda A)^{-1}(x_0)$, i.e.\
$\frac{1}{\lambda} (x_0 - x_\lambda) \in A(x_\lambda)$, then
$\frac{1}{\lambda} \ltwo{x_\lambda - x_0} \le \epsilon$ implies that
$\dist(0, A(x_\lambda)) \le \epsilon$.
Conversely, if $\dist(0, A(x_0)) \le \epsilon$, then using the
monotonicity of $A$ and letting $v_0 \in A(x_0)$ satisfy
$\ltwo{v_0} \le \epsilon$, the containment
$\frac{1}{\lambda} (x_0 - x_\lambda) = v_\lambda \in A(x_\lambda)$
implies that
\begin{align*}
  \ltwo{x_0 - x_\lambda}^2
  = \lambda \<x_0 - x_\lambda, v_\lambda\>
  \le \lambda \<x_0 - x_\lambda, v_0\>
  \le \lambda \ltwo{x_0 - x_\lambda} \epsilon
  ~~~ \mbox{so} ~~~
  \frac{1}{\lambda} \ltwo{x_0 - x_\lambda} \le \epsilon.
\end{align*}
%
%
Relatedly, the Moreau envelope
$F_\lambda(x) \defeq \inf_y \{F(y) + \frac{1}{2\lambda} \ltwo{x - y}^2\}$
of $F$ satisfies
\begin{align*}
  \nabla F_\lambda(x)
  = \frac{1}{\lambda} (x - \proxmap_{\lambda F}(x)).
\end{align*}
(See \cite{Moreau65} or \cite[Ch.~12.4]{BauschkeCo17}.)
So if $\ltwo{\nabla F_\lambda(x)} \le \epsilon$ then
$x^+ = \proxmap_{\lambda F}(x)$ satisfies
\begin{equation*}
  \ltwos{x^+ - x} \le \lambda \epsilon,
  ~~
  \dist(0, \partial F(x^+)) \le \epsilon,
  ~~ \mbox{and} ~~
  F(x^+) \le F(x),
\end{equation*}
and conversely, if $\dist(0, \partial F(x)) \le \epsilon$ then
$\ltwo{\nabla F_\lambda(x^+)} \le \epsilon$.

These implications motivate a substantial literature in stochastic and
non-stochastic optimization, with a special focus on weakly convex
functions, where researchers have provided quantitative guarantees on
the rates of convergence of
$\lambda^{-1} \ltwos{x_k - \proxmap_{\lambda \poploss_P}(x_k)}$ of
iterates $x_k$ from, for example, Gauss-Newton and prox-linear
methods, stochastic gradient methods, and stochastic model-based
minimization~\cite{DrusvyatskiyPa19, DavisGr19, DavisDr19, AsiDu19}.
\citet{DavisDr22} extend these techniques to demonstrate the empirical
convergence
\begin{equation*}
  \sup_{x \in X} \ltwo{\nabla (\poploss_P)_\lambda(x)
    - \nabla (\poploss_{P_n})_\lambda(x)} = O\left(\sqrt{\frac{d}{n}
    \log \frac{n \cdot \diam(X)}{\lambda}}\right),
\end{equation*}
of the Moreau envelopes of $\poploss_P$ and $\poploss_{P_n}$ with high
probability, which in turn implies~\cite[Thm.~5.1]{DavisDr22} the
graphical convergence that
$\dhaus(\gr \partial \poploss_{P_n}, \gr \partial \poploss_P) \to 0$.
%
Related approaches~\cite{Bianchi16, CombettesMa25} show convergence of
iterates of stochastic approximation-type methods to stationary
points, that is, they show results of the form
$\dist(x_k, \popop^{-1}(0)) \to 0$ in quite general settings.
These guarantees differ fundamentally from the stronger
notion~\eqref{eqn:stationary-point-goal} of stationarity we seek,
which requires procedures to return a point with small stationary
residual: they only ensure that a neighborhood of a chosen point
contains a nearly stationary point.
As the absolute value $F(x) = |x|$ makes clear, even for $1$-Lipschitz
functions $\ltwo{x - \proxmap_F(x)}$ can be arbitrarily small while
$\ltwos{\nabla F(x)} = 1$.




When the random operators $A_z$ are single-valued and appropriately
Lipschitz continuous, additional guarantees are possible.
Particular references include \citet{JuditskyNeTa11} and
\citet{KotsalisLaLi22}, who provide convergence guarantees for
stochastic mirror-prox and operator extrapolation methods.
Similarly, \citet{FosterSeShSrWo19} provide several results for
finding stationary points of smooth convex functions, with
finite-sample guarantees so long as $\nabla \loss_z$ is Lipschitz
continuous for each $z$.
Our results assume none of this structure.

\subsection{An exemplar challenge in finding nearly stationary points}
\label{sec:baby-challenges}

\newcommand{\separation}{\Delta}

To highlight the difficulty of finding stationary points, we find it
instructive to exhibit convex functions for which empirical risk minimizers
are non-stationary at all sample sizes.
We take \citet[Thm.~3]{TianRo25} as motivation and
construct an example with the property that $\poploss_{P_n}$
has an empirical minimizer whose population
gradient is large with probability 1.

Consider a countable collection
$\{u_j\}_{j \in \N} \subset \sphere^{d-1}$
satisfying the non-alignment condition
\begin{align}
  \label{eqn:non-alignment}
  \separation_i \defeq \sup\left\{ \<u_i, u_j\> \mid j \neq i\right\} < 1.
\end{align}
(See Appendix~\ref{sec:construct-non-alignment} for a construction.)
Then for $z \in \{0, 1\}^\N$, define the
$1$-Lipschitz function
\begin{equation}
  \label{eqn:bad-losses}
  \loss_z(x) \defeq \sup_{j \in \N} z_j \hinge{\<x, u_j\> - 1},
\end{equation}
which is convex, nonnegative, and $z$-measurable.
Empirical samples of these losses fail to have
uniform convergence of the subdifferential and
have highly non-stationary minimizers:
\begin{proposition}
  \label{proposition:bad-losses}
  Let $0 < \delta < 1$ and
  $Z \in \{0, 1\}^\N$ have i.i.d.\ $\bernoulli(1 - \delta)$ coordinates.
  For the $1$-Lipschitz convex losses~\eqref{eqn:bad-losses}, for each
  sample size $n$, with probability $1$ over the sample $P_n$
  there exists an empirical risk
  minimizer $\what{x}_n$ and vector $u_j \in \sphere^{d-1}$ for which
  \begin{equation*}
    \nabla \poploss_{P_n}(\what{x}_n) = 0
    ~~ \mbox{while} ~~
    \nabla \poploss_P(\what{x}_n) = (1 - \delta) u_j.
  \end{equation*}
\end{proposition}
\begin{proof}
  With probability 1,
  there exists a coordinate $j$ for which $P_n(Z_j = 0) = 1$.
  Along this coordinate, we have $\loss_z(t u_j) = z_j \hinge{t - 1}$
  for $0 \le t < 1 / \hinge{\separation_j}$, and
  for each $t u_j$ with $0 \le t < 1 / \hinge{\separation_j}$,
  there exists a neighborhood $U$ of $t u_j$
  such that all $x \in U$ satisfy
  $\loss_z(x)  = z_j \hinge{\<u_j, x\> - 1} = 0$.
  Evidently, $\poploss_{P_n}(x) = 0$ and $\nabla \poploss_{P_n}(x) = 0$ for
  such $x$.
  On the other hand, if $1 < t < 1 / \hinge{\separation_j}$,
  then on a small enough neighborhood $U$ of $t u_j$, the population
  loss $\poploss_P(x) = \E_P[Z_j] (\<u_j, x\> - 1)
  = (1 - \delta) (\<u_j, x\> - 1)$
  satisfies
  $\nabla \poploss_P(x) = (1 - \delta) u_j$.
\end{proof}

Because the functions~\eqref{eqn:bad-losses} are $1$-Lipschitz, the
non-stationarity Proposition~\ref{proposition:bad-losses} exhibits is
essentially as bad as possible.
Moreover, even though along each line $\{t u\}_{t \in \R}$, $u \in
\sphere^{d-1}$ the directional derivatives exhibit uniform convergence
$\sup_{t \in \R} |\poploss_{P_n}'(t u; u) - \poploss_P'(t u; u)| \cp 0$,
this convergence fails along many directions.
This suggests some of the difficulties in showing that ``standard''
algorithms, such as regularized empirical minimization or proximal point
methods, could generate iterates satisfying $\dist(0, \nabla
\poploss_P(\what{x}_n)) \cp 0$: any such analysis (we have tried many)
would need to rely on structural properties of convex functions beyond those
that are, at least to us, apparent.
Accordingly, we take a different approach.



\section{The main result}

\newcommand{\lebesgue}{\mathsf{Leb}}  
\newcommand{\statpoints}{S}  
\newcommand{\resolvent}{\proxmap}  

Our main technical result relies on the
proximal point mapping for the sampled
objective $\poploss_{P_n}$.
In brief, we show that the set of points  $x \in X$ for which
the proximal point update
\begin{align*}
  x_\lambda \defeq \proxmap_{\lambda \poploss_{P_n}}(x)
  = \argmin_{y \in X} \left\{\poploss_{P_n}(y)
  + \frac{1}{2 \lambda} \ltwo{y - x}^2 \right\}
\end{align*}
falls into the set of $\epsilon$-stationary points has
strictly positive (Lebesgue) measure for any $\lambda > 0$.
More precisely, denote the set of $\epsilon$-stationary points for
the population objective $\poploss_P$ by
\begin{align*}
  \statpoints_P(\epsilon) \defeq \left\{x \in X \mid
  \dist(0, \partial \poploss_P(x) + \normalcone_X(x)) \le \epsilon \right\}.
\end{align*}
Then as a corollary of Theorem~\ref{theorem:lebesgue-measure-positive} to
follow, we have the following convergence result,
where $\lebesgue$ denotes ($d$-dimensional) Lebesgue measure.
\begin{corollary}
  \label{corollary:lebesgue-function-case}
  Assume that each $\loss_z$ is
  $\lipconst(z)$-Lipschitz over $X$, where $\E_P[\lipconst^2(Z)] < \infty$,
  and that $X$ is compact.
  Let $\epsilon > 0$.
  %
  %
  Then there exists a positive $c > 0$, depending on $X$,
  $\epsilon$, $P$, and $\{\loss_z\}_{z \in \mc{Z}}$, for which
  \begin{equation*}
    \P\left(\lebesgue\left(\left\{x \in X + 4 \epsilon \ball
    \mid \proxmap_{\poploss_{P_n}}(x)
    \in \statpoints_P(\epsilon)\right\}\right) \ge c \right) \to 1.
  \end{equation*}
\end{corollary}

Corollary~\ref{corollary:lebesgue-function-case}
leads to a simple and naive convergent algorithm:
assume we have two independent samples $P_n$ and $P_n'$, drawn i.i.d.\
$P$.
Draw $u^{(i)}$ uniformly from $X + 4 \epsilon \ball$, $i = 1, \ldots, m$,
and set $x^{(i)} = \proxmap_{\poploss_{P_n}}(u^{(i)})$.
Then report the index $i$ minimizing
the empirical subdifferential on the second sample,
\begin{align*}
  \what{x}_n = x^{(\what{i})}
  ~~~ \mbox{for} ~~~
  \what{i}
  \defeq \argmin_i \dist\left(0,
  \partial \poploss_{P_n'}(x^{(i)}) + \normalcone_X(x^{(i)})
  \right).
\end{align*}
Because empirical subdifferentials converge pointwise (see
Lemma~\ref{lemma:empirical-subdifferential-pointwise} to come), a quick
argument via the union bound shows that, so long as $1 \ll m \ll n$, we have
\begin{align*}
  \dist\left(0, \partial \poploss_P(\what{x}_n)
  + \normalcone_X(\what{x}_n)\right)
  \cp 0.
\end{align*}
In the next section, we provide a slightly more sophisticated variant of
this procedure, filling out the details of the argument.

\subsection{The stochastic monotone operator case}

Our results extend beyond the convex function
case to more general (stochastic) maximal monotone operators, which
provide a natural setting that coincides with the analytic techniques
to prove the results.
As in Section~\ref{sec:related-work},
assume we have a collection of maximal monotone mappings
$A_z(\cdot) : \R^d \toto \R^d$ for $z \in \mc{Z}$
and $A_P(x) = \E_P[A_Z(x)]$ for $Z \sim P$.
%
We assume $A_P$ is maximal monotone over $X$, that is,
$A_P + \normalcone_X$ is maximal monotone.
We wish to find $x \in X$ solving
\begin{equation}
  \label{eqn:vi}
  0 \in A_P(x) + \normalcone_X(x),
\end{equation}
which we assume has solutions
$X\opt = \{x \in X \mid 0 \in A_P(x) + \normalcone_X(x)\}$.

We require moment bounds on $A_z(x)$, and so say that $A_z$ has $p \ge 2$
moments near $x^0$ if there exists $b > 0$ and envelope $\lipconst(z) <
\infty$ such that for $\ltwo{x - x^0} \le b$, each (measurable) selection
$a_z(x) \in A_z(x)$ satisfies $\ltwo{a_z(x)} \le \lipconst(z)$ and
$\E[\lipconst^p(Z)] < \infty$.
Summarizing, to state our results, we make the following standing
assumptions:
\begin{assumption}
  \label{assumption:basic-vi-setting}
  Let $p \ge 2$.
  The following hold:
  \begin{enumerate}[label=(\roman*),leftmargin=*]
  \item \label{item:vi-max-monotone}
    The population operator
    $A_P + \normalcone_X : \R^d \toto \R^d$ is maximal monotone.
  \item \label{item:vi-solution}
    $X$ has interior and the variational inequality~\eqref{eqn:vi}
    has non-empty solution set
    $X\opt \subset X$.
  \item \label{item:vi-moments}
    $A_z$ has $p$ moments near
    the minimum norm solution $x^0 = \argmin_{x \in X\opt} \ltwo{x}$.
  \end{enumerate}
\end{assumption}

The subdifferential of a convex objective provides the standard example
satisfying Assumption~\ref{assumption:basic-vi-setting}.
In this case, the subdifferential $A_z(x) = \partial \loss_z(x)$, and the
measurability assumptions and subdifferential equality $\partial
\poploss_P(x) = \E_P[\partial \loss_z(x)]$ follow essentially immediately,
so that maximal monotonicity holds~\cite{Rockafellar69, Bertsekas73,
  RockafellarWe98} and part~\ref{item:vi-max-monotone} follows.
Part~\ref{item:vi-solution}
trivially holds if $\poploss_P$ has a minimizer over $X$, and
then so long as
$\loss_z$ is $\lipconst(z)$-Lipschitz continuous near $x^0$
with $\E[\lipconst^2(Z)] < \infty$, part~\ref{item:vi-moments} also holds.
Other examples include the situation that
$X \subset \interior \dom A_z$ for each $z$ and $P$ has finite
support, which guarantees $A_P$ is maximal
monotone~\cite{Rockafellar70b}.
If $X$ is compact,
then $A_P + \normalcone_X$ is surjective~\cite[Corollary
  21.25]{BauschkeCo17} so~\eqref{eqn:vi} has a solution.

With this assumption, we give the natural generalization of stationary
points, letting
\begin{align*}
  \statpoints_P(\epsilon) = \{x \in X \mid \dist(0, A_P(x)
  + \normalcone_X(x)) \le \epsilon\}
\end{align*}
denote the $\epsilon$-stationary
points in the problem~\eqref{eqn:vi}.
Let $x^0$ denote the minimum norm solution to~\eqref{eqn:vi},
as in Assumption~\ref{assumption:basic-vi-setting}.\ref{item:vi-moments}.
We then have the following theorem.
\begin{theorem}
  \label{theorem:lebesgue-measure-positive}
  Let Assumption~\ref{assumption:basic-vi-setting}
  hold
  with neighborhood size $b$ in part~\ref{item:vi-moments}.
  Let $B = x^0 + b \cdot \ball$, $B' \supset B$ be closed convex, and
  $\epsilon > 0$, $\lambda > 0$.
  Then there exists
  $c = c(A, \epsilon, X, P, \lambda) > 0$ such that
  \begin{align*}
    \P\left(\lebesgue\left(
    \left\{x \in \R^d \mid \ltwo{x - x^0} \le 4 \lambda \epsilon,
    \proxmap_{\lambda A_{P_n} + \normalcone_{X \cap B'}}(x)
    \in \statpoints_P(\epsilon) \cap B \right\}\right)
    \ge c\right) \to 1.
  \end{align*}
\end{theorem}

A simple corollary is cleaner to state, though the more
general statement in Theorem~\ref{theorem:lebesgue-measure-positive}
allows us to demonstrate that an algorithm that localizes
around empirical minimizers converges in probability.
\begin{corollary}
  \label{corollary:lebesgue-measure-positive-simple}
  Let Assumption~\ref{assumption:basic-vi-setting} hold
  and $\epsilon > 0$.
  Then there is a positive $c > 0$ such that for all $r \ge 4 \epsilon$,
  \begin{align*}
    \P\left(\lebesgue\left(
    \left\{x \in \R^d
    \mid \ltwo{x - x^0} \le r, \proxmap_{A_{P_n} + \normalcone_X}(x)
    \in \statpoints_P(\epsilon)\right\}\right) \ge c \right) \to 1.
  \end{align*}
\end{corollary}

\subsection{Proof sketch}

The proof of Theorem~\ref{theorem:lebesgue-measure-positive}
reposes on a dimensionality analysis of the graph
of the monotone operator $A_P: \R^d \toto \R^d$
as a subset of $\R^d \times \R^d$.
As we shall see, for an appropriate notion of dimension, the graph has
dimension $d$, and by using the Minty parameterization of the graph, we may
divide the graph into pieces, at least one of which is $d$-dimensional.
Then via a careful analysis using geometric measure theory, we may prove
that the measure of any $d$-dimensional piece is appropriately preserved.

We provide the formal proof in
Section~\ref{sec:proof-lebesgue-measure-positive}, but include an overview
of the ideas here, along with the main technique we adopt from convex
analysis and monotone operator theory.
To begin, we let $A : \R^d \toto \R^d$ be a maximal monotone operator,
and let
\begin{align*}
  \gr A \defeq \left\{(x, v) \mid x \in \dom A, v \in A(x)\right\}
  \subset \R^d \times \R^d
\end{align*}
be its graph;
we will follow \citet[Chapter 12]{RockafellarWe98} to develop
the \emph{Minty parameterization} of this subset, which provides
a bi-Lipschitz homeomorphism between $\R^d$ and $\gr A$.
To that end, for $\lambda > 0$ define the proximal mapping, or resolvent,
and Yosida approximation of $A$ by
\begin{align*}
  \res{\lambda}{A} \defeq (I + \lambda A)^{-1}
  ~~~ \mbox{and} ~~~
  \yosida{\lambda}{A} \defeq (\lambda + A^{-1})^{-1}
  = \frac{1}{\lambda} (I - \res{\lambda}{A}).
\end{align*}
Then $\res{\lambda}{A}$ is $1$-Lipschitz, while $\yosida{\lambda}{A}$ is
$\sqrt{1 + \lambda^{-2}}$-Lipschitz continuous~\cite[Remark
  23.23]{BauschkeCo17}.
The \emph{Minty parameterization}~\cite[Thm.~12.15]{RockafellarWe98}
\begin{align}
  \label{eqn:minty-param}
  \minty{\lambda}{A}(x) \defeq
  \left(\res{\lambda}{A}(x), \yosida{\lambda}{A}(x)\right)
  ~~ \mbox{and} ~~
  \minty{\lambda}{A}^{-1}(y, v)
  \defeq y + \lambda v
\end{align}
parameterizes the graph $\gr A$ in that
\begin{align*}
  \gr A = \left\{\minty{\lambda}{A}(x) \mid x \in \R^d
  \right\}
\end{align*}
and
\begin{equation*}
  (y, v) \in \gr A
  ~~ \mbox{if and only if} ~~
  (y, v) = \minty{\lambda}{A}(x)
  ~ \mbox{for} ~
  x = y + \lambda v.
\end{equation*}
As $\minty{\lambda}{A}$ is $\sqrt{1 + \lambda^{-2}}$-Lipschitz and
$\minty{\lambda}{A}^{-1}$ is $\sqrt{1 + \lambda^2}$-Lipschitz,
this provides a bi-Lipschitz homeomorphism between
$\R^d$ and $\gr A$.

We divide
the proof of Theorem~\ref{theorem:lebesgue-measure-positive}
into four parts; the first two hold for any maximal monotone operator
$A$, while the last two specialize to the sampling case that we consider:
\begin{enumerate}[label=(S.\arabic*),leftmargin=*]
\item \label{item:minty-to-full-subset}
  Use the Minty parameterization to identify full-dimensional compact
  convex subsets
  $K$ of the
  pre-image $\res{\lambda}{A}^{-1}(\statpoints(\epsilon))$ of the
  $\epsilon$-stationary points,
  so $\res{\lambda}{A}(K) \subset \statpoints(\epsilon)$.
\item \label{item:full-dim-parameterization} Argue that the subset $G = \gr
  A \cap (K \times \R^d)$ of the graph of $A$ has a particular
  full-dimensional parameterization, in that there is a further (compact)
  $X_0 \subset K$ for which $X_0$ has dimension $(d - k)$ while $A(x)$
  contains a $k$-dimensional but ``large enough'' simplex for each $x \in
  X_0$.
\end{enumerate}

\noindent
Step~\ref{item:full-dim-parameterization} is the most technical part of the
argument and provides the key to convergence, as pieces of the graph
decompose into complementary subsets.
It also is the least quantitative, in that the measures of
$X_0$ and $A(x)$ are only guaranteed to be positive
and finite.
Research on singular sets of convex
functions~\citep{AlbertiAmCa92,AmbrosioCaSo93,Alberti94,AlbertiAm99}
investigates similar decompositions of $\gr A$ via the dimension of $A(x)$,
but the results do not quite apply in the ways we seek in
Step~\ref{item:full-dim-parameterization}.
We discuss and provide examples
in Section~\ref{sec:discuss-dim-parameters}
after the proof of the theorem, highlighting both the
challenges in providing quantitative guarantees
and relationship with prior work.

\begin{enumerate}[label=(S.\arabic*),leftmargin=*]
  \setcounter{enumi}{2}
\item \label{item:fiber-survival}
  Taking $A_P = \int A_z(\cdot) dP(z)$,
  we argue that most of the sets $A_P(x)$ from
  step~\ref{item:full-dim-parameterization} survive under sampling,
  that is, $A_{P_n}(x)$ still contains a large enough set.
\item \label{item:large-hausdorff} Argue that the $d$-dimensional Hausdorff
  measure, which we define formally in the sequel, of the empirical graph
  $G_n = \gr A_{P_n} \cap (\res{\lambda}{A_P}(K) \times \R^d)$ is positive
  because of part~\ref{item:fiber-survival}, and so the transformation
  $\minty{\lambda}{A_{P_n}}^{-1}(G_n) \subset \R^d$ has positive
  $d$-dimensional Lebesgue measure because the Minty parameterization is a
  bi-Lipschitz homeomorphism.
\end{enumerate}

Given steps \ref{item:minty-to-full-subset}--\ref{item:large-hausdorff}, we
can then argue that the empirical graph has good behavior:
by step~\ref{item:large-hausdorff}, the set
of points
\begin{align*}
  \minty{\lambda}{A_{P_n}}^{-1}(G_n)
  = \left\{x \mid \minty{\lambda}{A_{P_n}}(x) \in G_n \right\}
  \subset
  \res{\lambda}{A_{P_n} + \normalcone_X}^{-1}(\statpoints(\epsilon)).
\end{align*}
Because step~\ref{item:large-hausdorff}
demonstrates that $\lebesgue(\minty{\lambda}{A_{P_n}}^{-1}(G_n)) \ge c$
(for some constant $c > 0$) with high probability,
this then implies the result.


\section{Algorithms}

Our algorithms, which rely on
Theorem~\ref{theorem:lebesgue-measure-positive}, proceed in two phases: the
first uses a regularized empirical risk minimizer to find a point near the
minimum norm solution $x^0$ to the variational inequality~\eqref{eqn:vi}
(recall Assumption~\ref{assumption:basic-vi-setting}
part~\ref{item:vi-moments}).
Given such an initial point, we can then sample in a small neighborhood
around $\what{x}_n$, and Theorem~\ref{theorem:lebesgue-measure-positive}
shows that one of these points likely has small subdifferential, and we can
check this.

\newcommand{\regmult}{\alpha}  

To set the stage, observe that strong monotonicity implies there is a
unique solution $x^\regmult \in X$ to the problem
\begin{equation*}
  0 \in A_P(x) + \normalcone_X(x) + \regmult x.
\end{equation*}
A quick argument (see Appendix~\ref{sec:proof-minimum-norm-solution}) shows
that $x^\regmult \to x^0$ as $\regmult \downarrow 0$:
\begin{lemma}
  \label{lemma:minimum-norm-solution}
  Let $A+\normalcone_X$ be maximal monotone and $X\opt = (A +
  \normalcone_X)^{-1}(0)$ be non-empty, with $x^0 = \argmin_{x \in X\opt}
  \ltwo{x}$.
  Then $x^\regmult = (A + \regmult I + \normalcone_X)^{-1}(0)$ satisfies
  $x^\regmult \to x^0$ as $\regmult \downarrow 0$.
\end{lemma}
\noindent
Applying a standard convergence argument, this lemma in turn implies that
the empirical regularized solution $\what{x}^{\regmult} \in X$ uniquely solving
\begin{align}
  \label{eqn:l2-empirical-vi}
  0 \in A_{P_n}(x) + \regmult x + \normalcone_X(x),
  ~~~ \mbox{i.e.} ~~~
  \what{x}^\regmult = \proxmap_{\regmult^{-1} A_{P_n} + \normalcone_X}(0)
\end{align}
also converges to $x^0$ when $\regmult \to 0$.
We have the following proposition, which holds in arbitrary Hilbert spaces,
whose proof we defer to Appendix~\ref{sec:proof-sampled-vi-convergence}.
\begin{proposition}
  \label{proposition:sampled-vi-convergence}
  Let Assumption~\ref{assumption:basic-vi-setting}
  hold,
  and assume that $\regmult = \regmult_n$ in problem~\eqref{eqn:l2-empirical-vi}
  satisfies
  $n^{-1/2} \ll \regmult_n \ll 1$.
  Then $\what{x}^{\regmult_n} \cp x^0$ as $n \to \infty$.
  If additionally $p > 2$ in Assumption
  \ref{assumption:basic-vi-setting}.\ref{item:vi-moments}, then
  $\what{x}^{\regmult_n} \to x^0$ with probability 1 whenever $\regmult_n
  \to 0$ and $\sum_{n \ge 1} (\regmult_n \sqrt{n})^{-p} < \infty$.
\end{proposition}
\noindent
Proposition~\ref{proposition:sampled-vi-convergence} shows that if we have
access to a proximal point operator, then obtaining the ``classical''
guarantees of convergence, i.e., that $\what{x}_n$ is close to nearly
stationary point, is essentially immediate.

\subsection{The randomized sampling algorithm}

To actually guarantee stationarity, we arrive at a simple algorithm:

\algbox{
  \label{alg:resample}
  Randomized resampling procedure
}{
  \textbf{Input:} radius $b > 0$ for
  which Assumption~\ref{assumption:basic-vi-setting}\ref{item:vi-moments}
  holds at radius at least $2b$ \\
  \textbf{Sample:} Let $P_n$ and $P_n'$ be i.i.d.\ samples
  of size $n$
  from $P$ \\
  \textbf{Optimize:} Let $\what{x}^{\regmult_n}$ solve the empirical
  problem~\eqref{eqn:l2-empirical-vi} with $\regmult_n = n^{-1/4}$
  and define the random neighborhood
  \begin{equation*}
    \what{B} = \what{x}^{\regmult_n} + b \ball
  \end{equation*}
  \textbf{Resample:} for $i = 1, \ldots, \sqrt{n}$,
  \begin{enumerate}[label=\arabic*.,leftmargin=3em]
  \item draw $u^{(i)}$ uniformly from $\what{B}$
  \item set $x^{(i)} = \resolvent_{A_{P_n} + \normalcone_{X \cap \what{B}}}(u^{(i)})
    = (I + A_{P_n} + \normalcone_{X \cap \what{B}})^{-1}(u^{(i)})$
  \end{enumerate}
  \textbf{Return}
  $\what{x}_n = \argmin_{x^{(i)}} \dist(0, A_{P_n'}(x^{(i)}) +
    \normalcone_X(x^{(i)}))$
}

\begin{theorem}
  The output $\what{x}_n$ of Algorithm~\ref{alg:resample}
  satisfies
  \begin{align*}
    \dist\left(0, A_P(\what{x}_n) + \normalcone_X(\what{x}_n)\right)
    \cp 0.
  \end{align*}
\end{theorem}
\begin{proof}
  We begin with a containment lemma on pre-images of proximal mappings,
  demonstrating a useful containment.
  \begin{lemma}
    \label{lemma:resolvent-preimage-containment}
    Let $A$ be maximal monotone, $S$ a closed set,
    and $C_0 \subset C_1$ be closed convex sets, where
    $\dom A \supset (S \cap C_0)$.
    Then
    \begin{equation*}
      \resolvent_{A + \normalcone_{C_1}}^{-1}(S \cap C_0)
      \subset \resolvent_{A + \normalcone_{C_0}}^{-1}(S \cap C_0)
      \subset \resolvent_{A + \normalcone_{C_0}}^{-1}(S).
    \end{equation*}
  \end{lemma}
  \begin{proof}
    Let $x \in \resolvent_{A + \normalcone_{C_1}}^{-1}(S \cap C_0)$.
    Then $\resolvent_{A + \normalcone_{C_1}}(x) \in S \cap C_0$,
    that is, there
    exists $y \in S \cap C_0$ such that
    $x \in (I + A + \normalcone_{C_1})(y)$.
    Because $C_0 \subset C_1$, we have
    $\normalcone_{C_0}(y) \supset \normalcone_{C_1}(y)$
    for each $y \in C_0$.
    As $y \in S \cap C_0$,
    \begin{align*}
      x \in y + A(y) + \normalcone_{C_1}(y)
      \subset y + A(y) + \normalcone_{C_0}(y).
    \end{align*}
    Rewriting,
    $x \in (I + A + \normalcone_{C_0})(y)$ and
    $y = \resolvent_{A + \normalcone_{C_0}}(x)$,
    so $x \in \resolvent_{A + \normalcone_{C_0}}^{-1}(S \cap C_0)$.
  \end{proof}

  Proposition~\ref{proposition:sampled-vi-convergence} implies
  $\what{x}^{\regmult_n} \cp x^0$, so defining the sets $B_{1/2} = x^0 +
  \frac{b}{2} \ball$, $B_2 = x^0 + 2 b \ball$, and $\what{B} =
  \what{x}^{\regmult_n} + b \ball$,
  \begin{align*}
    \lim_{n \to \infty} \P(B_{1/2} \subset \what{B} \subset B_2) = 1.
  \end{align*}
  Let $0 < \epsilon \le b/4$ and recall the set $\statpoints(\epsilon) = \{x
  \in X \mid \dist(0, A_P(x) + \normalcone_X(x)) \le \epsilon\}$ of
  $\epsilon$-stationary points, which is closed by outer semicontinuity of
  $A_P$.
  Consider the random sets
  \begin{align*}
    X_{0,n}^- = \resolvent_{A_{P_n} + \normalcone_{X \cap B_{2}}}^{-1}
    \left(\statpoints(\epsilon)
    \cap B_{1/2}\right) \cap \what{B}
    ~~ \mbox{and} ~~
    X_{0,n} = \resolvent_{A_{P_n} + \normalcone_{X \cap \what{B}}}^{-1}
    \left(\statpoints(\epsilon)\right) \cap \what{B}.
  \end{align*}
  Then Lemma~\ref{lemma:resolvent-preimage-containment} implies that
  $X_{0,n}^- \subset X_{0,n}$, while applying
  Theorem~\ref{theorem:lebesgue-measure-positive} shows that because
  $\what{B} \supset B_{1/2}$ with probability tending to 1, there exists $c
  > 0$ such that
  \begin{align*}
    \P\left(\lebesgue(X_{0,n}^-) \ge c\right) \to 1,
    ~~~ \mbox{whence} ~~~
    \P\left(\lebesgue(X_{0,n}) \ge c\right) \to 1.
  \end{align*}
  Moreover, we see that on the event that $\lebesgue(X_{0,n}) \ge c$,
  we have
  \begin{align*}
    \P(u^{(i)} \not \in X_{0,n} \mid X_{0,n})
    = 1 - \frac{\lebesgue(X_{0,n})}{\lebesgue(X \cap \what{B})}
    \le 1 - \frac{c}{b^d \lebesgue(\ball)},
  \end{align*}
  and so for some $c' > 0$ independent of $n$, we have
  \begin{align}
    \label{eqn:random-sample-in-preimage}
    \P\left(u^{(i)} \not \in X_{0,n} ~ \mbox{for~each~} i =
    1, \ldots, \sqrt{n} \mid X_{0,n}\right)
    \le (1 - c')^{\sqrt{n}}.
  \end{align}
  Thus, defining the event
  \begin{equation*}
    \mc{E}_n
    \defeq
    \bigcup_{i = 1}^{\sqrt{n}} \left\{u^{(i)} \in X_{0,n}\right\}
  \end{equation*}
  that at least one of the sampled
  points $u^{(i)}$ belongs to the pre-image $X_{0,n} =
  \resolvent_{A_{P_n} + \normalcone_{X \cap \what{B}}}^{-1}(\statpoints(\epsilon))$,
  we have
  $\P(\mc{E}_n) \to 1$ as $n \to \infty$.

  We now work on the event $\mc{E}_n$.
  Cheating by looking ahead to use the pointwise convergence of sampled
  Hausdorff distances in
  Lemma~\ref{lemma:empirical-subdifferential-pointwise}, if
  $\E[\dhaus(\{0\}, A_z(x))^2] \le \lipconst^2$ for $x$ near $x^0$ then
  $\E[\dhaus^2(A_{P_n}(x), A_P(x))] \lesssim \frac{\lipconst^2 d}{n}$.
  Because $P_n'$ is an independent sample, an application of
  Chebyshev's inequality, this moment bound, and a
  union bound yields
  \begin{lemma}
    \label{lemma:worlds-dumbest-generalization-bound}
    Let $x^{(1)}, \ldots, x^{(m)} \in X$ be independent of
    $P_n'$
    and lie in the neighborhood of $x^0$
    Assumption~\ref{assumption:basic-vi-setting} specifies.
    Then for any $\epsilon > 0$,
    \begin{align*}
      \P\left(\max_{i \le m}
      \dhaus\left(A_P(x^{(i)}), A_{P_n'}(x^{(i)})\right) \ge \epsilon \right)
      \lesssim  \frac{\lipconst^2 d m}{n \epsilon^2}.
    \end{align*}
  \end{lemma}

  We use the lemma to complete the proof.
  By the triangle inequality,
  \begin{align*}
    \dist(0, A_P(\what{x}_n) + \normalcone_X(\what{x}_n))
    &
    \le \dhaus(A_P(\what{x}_n), A_{P_n'}(\what{x}_n))
    + \dist(0, A_{P_n'}(\what{x}_n) + \normalcone_X(\what{x}_n)) \\
    & \le \max_{i \le m} \dhaus(A_P(x^{(i)}), A_{P_n'}(x^{(i)}))
    + \min_{i \le m}
    \dist\left(0, A_{P_n'}(x^{(i)}) + \normalcone_X(x^{(i)})\right) \\
    & \le 2 \max_{i \le m} \dhaus\left(A_P(x^{(i)}), A_{P_n'}(x^{(i)})\right)
    + \min_{i \le m} \dist\left(0, A_P(x^{(i)}) + \normalcone_X(x^{(i)})\right).
  \end{align*}
  On the good event $\mc{E}_n$, we know that
  the final term is at most $\epsilon$,
  while Lemma~\ref{lemma:worlds-dumbest-generalization-bound}
  shows that so long as $\what{x}^{\regmult_n}$ is near $x^0$,
  $\dhaus(A_P(x^{(i)}), A_{P_n'}(x^{(i)})) \le \epsilon$
  for each $i$ with probability at least
  $1 - C \frac{1}{\sqrt{n} \epsilon^2} \to 1$.
  That is,
  \begin{align*}
    \P\left(\dist(0, A_P(\what{x}_n) + \normalcone_X(\what{x}_n)) \ge 3\epsilon
    \right) \to 0.
  \end{align*}
  Because $\epsilon > 0$ was arbitrary, this proves the theorem.
\end{proof}

\subsection{The convex minimization randomized sampling procedure}

In the convex case where $A_z(x)=\partial \loss_z(x)$ our 
Algorithm~\ref{alg:resample} becomes:

\algbox{\label{alg:proximal-minimization} Randomized resampling
  for finding approximate stationary points}{
  \textbf{Input:} radius $b > 0$ \\
  \textbf{Sample:} Let $P_n$ and $P_n'$ be i.i.d.\ samples of size
  $n$ from $P$
  \\
  \textbf{Optimize:} Let $\regmult_n = n^{-1/4}$ and set
  \begin{align*}
    \what{x}^{\regmult_n} = \argmin_{x \in X} \left\{ \poploss_{P_n}(x)
    + \frac{\regmult_n}{2} \ltwo{x}^2 \right\}
    ~~ \mbox{and} ~~
    \what{B} = \what{x}^{\regmult_n} + b \ball.
  \end{align*}
  \textbf{Resample:} for $i = 1, \ldots, \sqrt{n}$,
  \begin{enumerate}[label=\arabic*.,leftmargin=3em]
  \item draw $u^{(i)}$ uniformly from $\what{B}$
  \item set $x^{(i)} = \proxmap_{\poploss_{P_n}}(u^{(i)}) = \argmin_{x \in X
    \cap \what{B}} \{\poploss_{P_n}(x) + \half \ltwos{x - u^{(i)}}^2\}$
  \end{enumerate}
  \textbf{Return:}
  $\what{x}_n = \argmin_{x \in \{x^{(i)}\}} \dist(0, \partial \poploss_{P_n'}(x)
  + \normalcone_X(x))$
}

\begin{corollary}
  Let Assumption~\ref{assumption:basic-vi-setting} hold
  for $A_z = \partial \loss_z$.
  Then Algorithm~\ref{alg:proximal-minimization} satisfies
  \begin{align*}
    \dist\left(0, \partial \poploss_P(\what{x}_n) + \normalcone_X(\what{x}_n)
    \right)
    \cp 0.
  \end{align*}
\end{corollary}



\section{Proof of Theorem~\ref{theorem:lebesgue-measure-positive}}
\label{sec:proof-lebesgue-measure-positive}

\newcommand{\inddim}{\mathop{\textup{ind}}}  
\newcommand{\Inddim}{\mathop{\textup{Ind}}}  
\newcommand{\dimhaus}{\textup{dim}_{\hausmeasure}}  
\newcommand{\boundconst}{r}  

This section provides the full proof of
Theorem~\ref{theorem:lebesgue-measure-positive}.
We follow the outline in
steps~\ref{item:minty-to-full-subset}--\ref{item:large-hausdorff} in
Sections~\ref{sec:proof-minty-to-full-subset} through
\ref{sec:proof-large-hausdorff}.
Before continuing, however, we provide some preliminary
background on convergence of empirical subdifferentials (and monotone
operators) to their population counterparts, as well as a brief
review of the dimension theory we will use.

\subsection{Convergence of empirical subdifferentials and monotone maps}

We begin with guarantees on the (pointwise) convergence of subdifferentials
of convex functions and, more generally, set-valued random functions.
Recall the Hausdorff distance
\begin{equation*}
  \dhaus(A, B) =
  \max\Big\{\sup_{a \in A} \dist(a, B), \sup_{b \in B} \dist(b, A)\Big\}
\end{equation*}
between
sets, where $\dist(x, A) = \inf_{a \in A} \ltwo{x - a}$.
Now, let $S_z \subset \R^d$ be compact-convex-valued sets, measurable
in $z$, and let $S_P = \E_P[S_Z]
= \int S_z dP(z)$ be the (Aumann) expectation of $S_Z$
for $Z \sim P$
and $S_{P_n} = \frac{1}{n} \sum_{i = 1}^n S_{Z_i}$ its empirical counterpart.
\begin{lemma}
  \label{lemma:empirical-subdifferential-pointwise}
  Let $2 \le p < \infty$
  and assume $\sup\{\ltwo{v} \mid v \in S_z\}
  \le \lipconst(z)$, where $\E[\lipconst^p(Z)] < \infty$.
  Then for a constant $C_p$ depending only on $p$,
  \begin{align*}
    \E_P\left[\dhaus^p(S_{P_n}, S_P)\right]
    \le C_p \cdot \frac{d^{p/2} \cdot \E[\lipconst^p(Z)]}{n^{p/2}}.
  \end{align*}
\end{lemma}
\noindent
See Appendix~\ref{sec:proof-empirical-subdifferential-pointwise} for a
proof.
The lemma also has the converse that there exist $1$-Lipschitz convex loss
functions for which the empirical subdifferential has expected Hausdorff
distance from the population subdifferential at least $\min\{1,
\sqrt{d/n}\}$, though we omit this result.

Lemma~\ref{lemma:empirical-subdifferential-pointwise} extends to the
monotone operator cases we consider: if $A_z : \R^d \toto \R^d$
satisfy Assumption~\ref{assumption:basic-vi-setting},
then we immediately obtain the following result.
\begin{lemma}
  \label{lemma:empirical-subdifferential-pointwise-normal}
  Let Assumption~\ref{assumption:basic-vi-setting} hold.
  Then for $x$ near $x^0$,
  \begin{align*}
    \E\left[\dhaus^2(A_{P_n}(x) + \normalcone_X(x),
      A_P(x) + \normalcone_X(x))\right]
    \le O(1) \cdot \frac{d \cdot \E[\lipconst^2(Z)]}{n}.
  \end{align*}
\end{lemma}
\begin{proof}
  If $C_0, C_1, K \subset \R^d$, for
  any
  $x \in C_0$ and $y \in K$,
  we have $\dist(x + y, C_1 + K)
  \le \dist(x, C_1)$.
  Thus
  $\dhaus(C_0 + K, C_1 + K) \le \dhaus(C_0, C_1)$.
  Apply Lemma~\ref{lemma:empirical-subdifferential-pointwise}.
\end{proof}

In some cases, we will truncate the random monotone operators
we consider.
Accordingly, we have the following result,
whose proof we defer to Appendix~\ref{sec:proof-truncating-hausdorff}.
\begin{lemma}
  \label{lemma:truncating-hausdorff}
  Let $C$ and $D$ be closed convex sets,
  and let $c > 0$ be large enough that the ball
  $B = c \cdot \ball$ satisfies
  $D \cap B \neq \emptyset$.
  Then
  \begin{align*}
    \dhaus(C, D) \le \frac{c}{2}
    ~~ \mbox{implies} ~~
    \dhaus(C \cap 2B, D \cap 2B)
    \le 6 \cdot \dhaus(C, D).
  \end{align*}
\end{lemma}

For any $t > 0$, in the setting of the lemma, $\dhaus(C \cap 2B, D \cap 2B)
\ge t$ therefore implies that either $\dhaus(C, D) \ge \frac{c}{2}$ or
$\dhaus(C, D) \ge t/6$.
Thus, if $S_n$ is a random sequence of convex sets
and $B = c \cdot \ball$ satisfies $S \cap B \neq \emptyset$,
we obtain for any $p < \infty$ and $k \ge 0$ that
\begin{align*}
  \E\left[\dhaus^p(S_n \cap 2B, S \cap 2B) \wedge k \right]
  & = \int_0^{k^{1/p}} \P(\dhaus(S_n \cap 2B, S \cap 2B) \ge t^{1/p})
  dt \\
  & \le \int_0^{k^{1/p}}
  \left(\P(\dhaus(S_n, S) \ge t^{1/p} / 6)
  + \P(\dhaus(S_n, S) \ge c/2)\right) dt \\
  & \le \E\left[6^p \dhaus^p(S_n, S) \wedge k\right]
  + k^{1/p} \cdot \P(\dhaus(S_n, S) \ge c/2).
\end{align*}
The following lemma is then immediate.
\begin{lemma}
  \label{lemma:convergence-after-truncation}
  Let $S_n$ and $S \subset \R^d$ be convex sets satisfying $\dhaus(S_n, S)
  \cp 0$, and let $c > 0$ be large enough that the ball $B = c \cdot \ball$
  satisfies $S \cap B \neq \emptyset$.
  Then
  \begin{align*}
    \dhaus(S_n \cap 2 B, S \cap 2 B) \cp 0.
  \end{align*}
  If $\E[\dhaus^p(S_n, S)] \to 0$ for some $p > 0$, then $\E[\dhaus^p(S_n
    \cap 2B, S \cap 2B) \wedge k] \to 0$ for any $k < \infty$.
\end{lemma}

\subsection{A brief review of dimension theory}

We provide definitions relating to dimension and Hausdorff measure
that play a role in our arguments.
For background, see the books on dimension theory
by~\citet{HurewiczWa41} and \citet{Engelking78}, and
geometric measure theory by \citet{Federer69} and \citet{Maggi12}.

Defining the dimensional constant $\omega_s \defeq
\frac{\pi^{s/2}}{\Gamma(\frac{s}{2} + 1)}$ for $s \ge 0$, the $s$-Hausdorff
measure of a set $X$ at diameter $\delta$ is
\begin{align*}
  \hausmeasure^s_\delta(X)
  \defeq \inf\left\{\sum_{i = 1}^\infty \omega_s
  \left(\frac{\diam(U_i)}{2} \right)^s
  \mid \diam(U_i) < \delta,
  \bigcup_i U_i \supset X
  \right\}
\end{align*}
the infimum taken over covers of $X$ with diameters less than $\delta$.
Because the infima are over smaller collections,
$\hausmeasure^s_\delta$ is non-decreasing as $\delta \downarrow 0$,
so one defines
\begin{align*}
  \hausmeasure^s(X)
  = \sup_{\delta > 0} \hausmeasure_\delta^s(X) = \lim_{\delta \downarrow 0}
  \hausmeasure_\delta^s(X).
\end{align*}
The $0$-dimensional Hausdorff measure $\hausmeasure^0(X) = \card(X)$.
%
For Borel $X \subset \R^d$, we always have $\hausmeasure^d(X) =
\lebesgue(X)$, where $\lebesgue$ denotes Lebesgue measure.
The
\emph{Hausdorff dimension} of a set is
\begin{align*}
  \dimhaus(X) \defeq \inf\left\{s \ge 0 \mid \hausmeasure^s(X) = 0 \right\}.
\end{align*}
For convex sets $X$, this coincides with the usual affine
dimension of the set~\cite[Thm.~6.2]{Rockafellar70},
and if $X \subset \R^d$ has interior, then
$\dimhaus(X) = d$~\cite[Ch.~3.1]{Maggi12}.

We also use the \emph{little inductive dimension} $\inddim$, where
recalling~\cite{HurewiczWa41, Engelking78}, $\inddim X$ has the inductive
definition that
\begin{enumerate}[leftmargin=*,label=(\roman*)]
\item $\inddim \emptyset = -1$
\item $\inddim X \le d$ if for all points $x \in X$
  and all neighborhoods $V \ni x$, there exists
  a neighborhood $U$ of $x$ with $x \in \cl U \subset V$ and
  $\inddim \bd U \le d - 1$, where
  $\bd U = \cl U \setminus U$.
\item $\inddim X = d$ if $\inddim X \le d$ and $\inddim X > d - 1$.
\item $\inddim X = +\infty$ if there is no $d$ for which $\inddim X \le d$.
\end{enumerate}
Standard results~\cite{HurewiczWa41, Engelking78} imply that $\inddim \R^d =
d$, and we always have $\inddim X \le \dimhaus X$ for all sets
$X$~\cite[Ch.~VII]{HurewiczWa41}.
The inductive dimension coincides with the topological dimension
(the Lebesgue covering dimension) when $X$ is a separable metric
space~\cite[Thm.~1.7.7]{Engelking78}.
As we work exclusively with separable metric spaces (subsets of $\R^d$ or
$\R^{2d}$ inheriting the Euclidean metric structure), these notions of
dimension are equivalent.

\subsection{Step~\ref{item:minty-to-full-subset}: the pre-image of stationary
  points}
\label{sec:proof-minty-to-full-subset}

We first consider a generic maximal monotone operator $A$ with
stationary points, where the Minty parameterization demonstrates
that subsets of the graph $\gr A$ have Hausdorff and inductive dimension
$d$.
Let $A : \R^d \toto \R^d$ have a solution $x_0$ with $0 \in A(x_0)$,
and for $\epsilon > 0$ let
\begin{align*}
  \statpoints(\epsilon) \defeq \left\{x \mid \dist(0, A(x)) \le \epsilon
  \right\}.
\end{align*}
To state the result, recall the Minty
parameterization~\eqref{eqn:minty-param} of the graph $\gr A$.

\begin{lemma}
  \label{lemma:minty-to-full-subset}
  Let $0 < \alpha < \epsilon$,
  $\lambda > 0$ and define the compact
  $K = x_0 + \lambda \alpha \ball$,
  where $0 \in A(x_0)$.
  Then the following hold.
  \begin{enumerate}[leftmargin=*,label=(\roman*)]
  \item \label{item:resolvent-near-stationary} The resolvent mapping
    satisfies $\res{\lambda}{A}(K) \subset \statpoints(\epsilon) \cap K$.
  \item \label{item:graph-dimension}
    For any $\epsilon \le \boundconst < \infty$,
    the graph piece
    \begin{align*}
      G \defeq \gr A \cap (\res{\lambda}{A}(K)
      \times \boundconst \cdot \ball) \supset
      \minty{\lambda}{A}(K)
    \end{align*}
    has Hausdorff, topological, and inductive dimension $d$.
  \item \label{item:coord-projection-stationary}
    Define the coordinate projection
    $\pi : \R^d \times \R^d \to \R^d$ by $\pi(x, z) = x$.
    Then
    \begin{align*}
      \pi(G) \subset S(\epsilon).
    \end{align*}
  \end{enumerate}
\end{lemma}
\begin{proof}
  We first prove part~\ref{item:resolvent-near-stationary}.
  Let $K = x_0 + \lambda \alpha \ball$
  and $x \in K$.
  Then because $\dom \res{\lambda}{A} = \R^d$~\cite[Thm.~12.1]{RockafellarWe98},
  there exist $(y, v) = \minty{\lambda}{A}(x) \in \gr A$ satisfying
  \begin{align*}
    y = (I + \lambda A)^{-1} (x)
    ~~ \mbox{and} ~~
    \lambda v = x - (I + \lambda A)^{-1}(x).
  \end{align*}
  Noting that $x_0 = (I + \lambda A)^{-1} (x_0)$, let $0 = u_0 \in A(x_0)$
  and $u \in A(y)$ satisfy $x_0 = x_0 + \lambda u_0$ and $x = y + \lambda
  u$, then use that $A$ is maximal monotone to obtain that
  \begin{align*}
    \frac{1}{\lambda} \ltwo{x_0 - y}^2
    & = \frac{1}{\lambda} \<I x_0 - I y, x_0 - y\> \\
    & \le \frac{1}{\lambda} \< I x_0 + \lambda u_0 - I y - \lambda u,
    x_0 - y\>
    = \frac{1}{\lambda} \<x_0 - x, x_0 - y\>.
  \end{align*}
  Applying
  the Cauchy-Schwarz inequality
  yields $\ltwo{y - x_0} \le \ltwo{x - x_0} \le \lambda \alpha$.
  Moreover,
  \begin{align*}
    \ltwo{v} = \ltwo{\yosida{\lambda}{A}(x)}
    = \frac{1}{\lambda} \ltwo{x - y} \le \alpha,
  \end{align*}
  so that $\ltwo{v} \le \alpha < \epsilon$.
  By definition of the resolvent, $v = Y_{\lambda A}(x) \in A(y)$,
  so $y \in S(\epsilon)$.

  For statements~\ref{item:graph-dimension}
  and~\ref{item:coord-projection-stationary} of the lemma, recall that
  $\minty{\lambda}{A}$ is a bi-Lipschitz homeomorphism and so preserves
  Hausdorff dimension, inductive dimension, and thus topological dimension
  as well.
  Then use the first part of the lemma and
  that $\dim K = d$.
\end{proof}

\subsection{Step~\ref{item:full-dim-parameterization}: full
  dimensional graph subsets}

The second step of the proof demonstrates that if $A$ is maximal monotone, we
may decompose the graph $\gr A$ into pieces that maintain
full-dimensionality.
Let $0 < \boundconst < \infty$ be a (constant) radius of interest, $X
\subset \R^d$ be a compact subset of $\dom A$, and define the
\emph{graph piece}
\begin{align*}
  G(X, \boundconst) \defeq \gr A \cap \left(X \times \boundconst \ball\right).
\end{align*}
Because $A$ is maximal monotone, $\gr A$ is closed,
so $G(X, \boundconst)$ is a compact set.
Define the projection $\pi : G \to \R^d$, $\pi(x, v) \defeq x$, onto the
first $d$ coordinates, and the slices (fibers)
\begin{align*}
  G_x \defeq A(x) \cap \boundconst \ball
  ~~ \mbox{and} ~~
  \pi^{-1}(x) = \{x\} \times G_x,
\end{align*}
where $G_x$ is compact convex because $A$ is maximal monotone and $\boundconst <
\infty$.
The following lemma shows that
we can decompose pieces of $G \subset \R^d \times \R^d$ into parts
of the form $(x, G_x)_{x \in X_k}$, where
each slice $G_x$ is $k$-dimensional,
and for at least one $k$, the base set $X_k$ is $(d - k)$-dimensional.

\begin{lemma}
  \label{lemma:splitting-hausdorff-slices}
  Let $A$ be maximal monotone, $X \subset \dom A$ be compact,
  $G \defeq G(X, \boundconst)$, and $G_x = A(x)
  \cap \boundconst \cdot \ball$ be non-empty for $x \in X$.
  Assume $G$ has inductive dimension $\inddim G = d$ (and hence Hausdorff
  measure $0 < \hausmeasure^d(G) < \infty$).
  Then there is a partition of $G$ into sets
  $K_i$ with following properties.
  \begin{enumerate}[label=(\roman*),leftmargin=*]
  \item \label{item:partition-by-hausdorff}
    For each $i = 0, 1, 2, \ldots, d$,
    \begin{align*}
      K_i = \bigcup_{x \in X} \left\{(x, v) \in G
      \mid 0 < \hausmeasure^i(G_x) < \infty
      \right\}
      = \bigcup_{x \in X} \left\{(x, v) \in G
      \mid \dimhaus G_x = i \right\}.
    \end{align*}
  \item \label{item:hausdorff-partition-measure}
    $\sum_{i = 0}^d \hausmeasure^d(K_i) = \hausmeasure^d(G)$.
  \item \label{item:base-graph-set}
    There is a $k \in \{0, 1, \ldots, d\}$ and $\epsilon_0 > 0$ such that
    for all $0 < \epsilon \le \epsilon_0$,
    the set
    \begin{align*}
      X_{k, \epsilon} \defeq \left\{
      x \in X \mid G_x ~ \mbox{contains~a~} k\mbox{-simplex}~
      \simplex_k ~ \mbox{with~} \epsilon \le \hausmeasure^k(\simplex_k)
      \right\}
    \end{align*}
    is compact, satisfies $\inddim X_{k,\epsilon}
    = \dimhaus X_{k,\epsilon} = d - k$, and $0 <
    \hausmeasure^{d-k}(X_{k,\epsilon}) < \infty$.
  \item \label{item:special-base-graph}
    Let $\pi : G \to \R^d$ be the coordinate projection $\pi(x, v) = x$.
    For the $k$ in part~\ref{item:base-graph-set},
    the set
    \begin{align*}
      X_k
      \defeq \left\{x \in X \mid 0 < \hausmeasure^k(G_x) < \infty\right\}
      = \pi(K_k)
    \end{align*}
    satisfies $0 < \lim_{\epsilon \downarrow 0}
    \hausmeasure^{d-k}(X_{k,\epsilon}) = \hausmeasure^{d-k}(X_{k})$
    and $\dimhaus X_k = d - k$.
  \end{enumerate}
\end{lemma}
\begin{proof}
  For each $x \in X$, the image $G_x = A(x) \cap \boundconst \ball \subset
  \R^d$ is a compact convex set, and so its topological dimension, affine
  dimension, and Hausdorff dimension
  coincide~\cite[Thm.~6.2]{Rockafellar70}.
  Thus for each $x \in X$, there exists some $i \in \{0,
  \ldots, d\}$ for which $0 < \hausmeasure^i(G_x) < \infty$,
  which coincides with the
  dimension of $G_x$ (for any of our notions of dimension).
  As $\hausmeasure^{i - \epsilon}(B) = +\infty$ while $\hausmeasure^{i +
    \epsilon}(B) = 0$ whenever $0 < \hausmeasure^i(B) < \infty$ for any $B
  \subset \R^d$~\cite[Ch.~3.1]{Maggi12},
  we see that
  the $K_i$ in part~\ref{item:partition-by-hausdorff} indeed
  partition $G$, and
  the dimensional equivalence is immediate.

  The second part~\ref{item:hausdorff-partition-measure} follows
  because the condition $\dimhaus G_x \ge k$ is equivalent to the existence
  of $k + 1$ affinely independent points in $G_x$, which yields Borel
  sets.
  Thus the Hausdorff measure is additive~\cite[Ch.~2]{Maggi12}.

  For the claim~\ref{item:base-graph-set},
  for $\epsilon > 0$ and $k \in \{0, 1, \ldots, d\}$,
  define the sets
  \begin{equation*}
    K_{k,\epsilon} = \left\{
    (x, v) \in G \mid ~ \mbox{for~some~}k~\mbox{simplex}~\simplex_k,
    G_x \supset \simplex_k ~ \mbox{and} ~
    \hausmeasure^k(\simplex_k) \ge \epsilon
    \right\}.
  \end{equation*}

  We make a simple observations about the compactness of
  the sets $K_{k,\epsilon}$.
  \begin{observation}
    \label{obs:Kk-compact}
    Each of the sets $K_{k,\epsilon}$ is compact.
  \end{observation}
  \begin{proof}
    Let
    $(x_i, v_i) \in K_{k,\epsilon}$ satisfy $(x_i, v_i) \to (x, v)$.
    Then by definition of $K_{k,\epsilon}$, we see that there exists a
    simplex $\simplex_k^{(i)} = \conv\{u_{0,i}, \ldots, u_{k,i}\} \subset
    G_{x_i}$ with $\hausmeasure^k(\simplex_k^{(i)}) \ge \epsilon$, and
    passing to a subsequence if necessary, there exist $u_j$ for which
    $\lim_i u_{j,i} = u_j \in G_x$ by outer semicontinuity of the mapping $x
    \mapsto A(x)$, and this outer semi-continuity implies that $v \in A(x)
    \cap \boundconst \ball$.
    The continuity of $\hausmeasure^k(\conv\{u_0, \ldots, u_k\})$ with
    respect to the vertices $u_j$ implies that $\hausmeasure^k(\conv\{u_0,
    \ldots, u_k\}) \ge \epsilon$, that is, $(x, v) \in K_{k,\epsilon}$.
    So $K_{k,\epsilon} \subset G$ is closed, hence ($G$ is compact because
    monotone operators have closed graphs) is compact.
  \end{proof}

  
  Fix any $k \in \{0, \ldots, d\}$ and $\epsilon > 0$, and let $K =
  K_{k,\epsilon}$ for this $\epsilon$.
  Recall the projection $\pi : G \to \R^d$ with $\pi(x, v) = x$ and
  $\pi^{-1}(x) = \{x\} \times G_x$.
  Then the Fubini-like result~\cite[Theorem~2.10.25 and
    Corollary~2.10.27]{Federer69} implies
  \begin{align*}
    \int_{\pi(K)}^* \hausmeasure^k(G \cap \pi^{-1}(x)) d\hausmeasure^{d-k}(x)
    \le C_{k, d-k} \hausmeasure^d(K)
  \end{align*}
  for a constant $C_{k,d-k}$ depending only on $d$ and $k$.
  Because $\hausmeasure^k(G \cap \pi^{-1}(x)) \ge \epsilon$
  by construction for each $x \in \pi(K)$,
  we thus obtain
  \begin{align}
    \epsilon \hausmeasure^{d-k}(\pi(K))
    \le \int_{\pi(K)}^* \hausmeasure^k(G \cap \pi^{-1}(x))
    d\hausmeasure^{d-k}(x)
    \le C_{k, d-k} \hausmeasure^d(K).
    \label{eqn:projected-K-finite}
  \end{align}
  As $\pi$ is Lipschitz continuous and $K$ is compact, $X_{k,\epsilon}
  \defeq \pi(K)$ is thus compact,
  and
  \begin{align*}
    \hausmeasure^{d-k}(X_{k,\epsilon})
    \le C_{k,d-k} \hausmeasure^{d}(K) / \epsilon
    < \infty.
  \end{align*}
  This yields the compactness and upper
  bound $\hausmeasure^{d-k}(X_{k,\epsilon}) < \infty$
  of part~\ref{item:base-graph-set} of the lemma.
  
  Now we leverage results on dimension lowering and raising mappings; see
  \citet[Ch.~4.3]{Engelking78} (though we use little-inductive instead of
  big-inductive dimension, because we work in separable metric spaces so
  that they are equivalent).
  We require additional notation.
  A mapping $h : \mc{Y} \to \mc{X}$ between separable metric spaces $\mc{Y}$
  and $\mc{X}$ is \emph{finite-dimensional} if
  \begin{align*}
    \Inddim h
    \defeq \inf\Big\{k \mid \sup_{x \in \mc{X}} \inddim h^{-1}(x) \le k\Big\}
    < \infty.
  \end{align*}
  Define the subsets $D_i(h) \defeq \{x \in \mc{X} \mid \inddim h^{-1}(x)
  \ge i\}$, and let
  \begin{align*}
    B(h) \defeq \begin{cases}
      \max\{ \inddim D_i(h) + i \mid 1 \le i \le \Inddim h
      \} & \mbox{if~} \Inddim h \ge 1 \\ -1 & \mbox{otherwise}
    \end{cases}
  \end{align*}
  be the biggest ``corrected'' dimension.
  The following result~\cite[Thm.~4.3.9]{Engelking78} holds, with a minor
  notational change:
  \begin{lemma}[Va\u{\i}n\v{s}te\u{\i}n's first theorem]
    \label{lemma:dimension-lowering}
    Let $\mc{Y}$ and $\mc{X}$ be separable metric spaces and
    $h : \mc{Y} \to \mc{X}$ be a finite-dimensional closed mapping.
    Then $\inddim \mc{Y} \le \max\{\inddim \mc{X}, B(h)\}$.
  \end{lemma}

  For $\pi : G \to \R^d$ the projection mapping
  and $\epsilon \ge 0$,
  define $X_{i,\epsilon} = \pi(K_{i,\epsilon})$.
  Because $G$ is compact and $\pi$ is continuous, it is closed, and
  for each $x \in X$
  we have $\pi^{-1}(x) = \{x\} \times G_x = \{x\}
  \times A(x) \cap \boundconst \cdot \ball$, which (by definition)
  has $\inddim \pi^{-1}(x) = \dimhaus \pi^{-1}(x) \le d < \infty$.
  So $\pi$ is finite dimensional, and
  by Lemma~\ref{lemma:dimension-lowering} (take $h = \pi$,
  $\mc{Y} = G$ and $\mc{X} = X$),
  we observe that
  \begin{align}
    \label{eqn:inductive-dimension-lower-bound-pieces}
    d = \inddim G \le \max\left\{\inddim X,
    B(\pi)\right\}.
  \end{align}

  We consider two cases in
  inequality~\eqref{eqn:inductive-dimension-lower-bound-pieces}: whether
  $\hausmeasure^d(X) > 0$ or $\hausmeasure^d(X) = 0$.
  If $\hausmeasure^d(X) > 0$, then evidently $\lebesgue^d(X) =
  \hausmeasure^d(X) > 0$, and so $\dom A$ has interior; it is therefore
  single-valued almost everywhere (see~\cite[Thm.~21.27]{BauschkeCo17} or
  \cite[Corollary 2]{Kenderov76}), and
  $X_{0,1} = X$
  satisfies the conditions of step~\ref{item:base-graph-set}.
  Otherwise, if $\hausmeasure^d(X) = 0$, then the volume lower
  bound~\cite[Thm.~VII.2]{HurewiczWa41} that $\inddim X \ge k$ implies
  $\hausmeasure^k(X) > 0$ for any $k \in \N$ guarantees that  
  $\inddim X \le d - 1$.
  To apply inequality~\eqref{eqn:inductive-dimension-lower-bound-pieces}
  recall the sets
  \begin{align*}
    D_k(\pi) \defeq \left\{x \in X \mid \inddim \pi^{-1}(x) \ge k \right\}
    = \left\{x \in X \mid \inddim G_x \ge k \right\}
    = \left\{x \in X \mid \hausmeasure^k(G_x) > 0 \right\},
  \end{align*}
  where the last equality follows because $G_x$ is convex compact, so that
  the dimensional definitions coincide~\cite[Thm.~6.2]{Rockafellar70}.
  Inequality~\eqref{eqn:inductive-dimension-lower-bound-pieces}
  thus implies that for some $k \in \{1, \ldots, d\}$,
  \begin{align*}
    d \le \inddim D_k(\pi) + k
    = \inddim \left\{x \in X \mid \hausmeasure^k(G_x) > 0 \right\} + k,
  \end{align*}
  i.e.,
  $\inddim D_k(\pi) \ge d - k$.
  The sum theorem~\cite[Thm.~1.5.3]{Engelking78} states that if $F$ is a
  separable metric space, $\{F_m\}$ is a countable collection of closed
  sets, and $F = \cup_{m \in \N} F_m$, then $\inddim F \le \sup_m \inddim
  F_m$.
  We then observe that
  \begin{align*}
    D_k(\pi) = \{x \in X \mid \hausmeasure^k(G_x) > 0\}
    = \bigcup_{\epsilon \in \Q_{> 0}} \pi(K_{k,\epsilon}),
  \end{align*}
  and the sets on the right are all closed
  (Observation~\ref{obs:Kk-compact}).
  So $\inddim \pi(K_{k,\epsilon}) \ge d - k$ for some $\epsilon > 0$
  and, in fact, for all small enough $\epsilon > 0$.
  Take any such $\epsilon$ and set $X_{k,\epsilon} = \pi(K_{k,\epsilon})$;
  again use the volume lower bound~\cite[Thm.~VII.2]{HurewiczWa41} to see
  that $\hausmeasure^{d-k}(X_{k,\epsilon}) > 0$.
  %
  The dimensional equality follows because
  $\inddim \le \dimhaus$.

  For the final claim~\ref{item:special-base-graph}, recall that that for
  $k, \gamma > 0$ with $k - \gamma \ge 0$, $\hausmeasure^{k - \gamma}(S) =
  \infty$ whenever $\hausmeasure^k(S) > 0$ and $\hausmeasure^{k + \gamma}(S)
  = 0$ if $\hausmeasure^k(S) < \infty$.
  To show that $\dimhaus X_k = d - k$, it is thus sufficient to
  demonstrate that $X_k = \pi(K_k)$ satisfies $\hausmeasure^{d - k}(X_k) >
  0$ and $\hausmeasure^{d - k + \gamma}(X_k) = 0$ for $\gamma > 0$.
  By~\cite[Thm.~2.10.25]{Federer69}, for all $\epsilon > 0$ we have
  $\int_{X_{k,\epsilon}}^* \hausmeasure^k(G_x) d \hausmeasure^{d-k}(x) \le
  C_{k,d-k} \hausmeasure^d(K_{k,\epsilon}) \le C_{k,d-k} \hausmeasure^d(G) <
  \infty$, so $\hausmeasure^{d-k}(\{x \in X_{k,\epsilon} \mid
  \hausmeasure^k(G_x) = +\infty\}) = 0$.
  By continuity of measure,
  \begin{align*}
    \hausmeasure^{d-k}(\{x \in X_k \mid \hausmeasure^k(G_x) = +\infty\}) =
    \lim_{\epsilon \downarrow 0} \hausmeasure^{d-k}(\{x \in X_{k,\epsilon}
    \mid \hausmeasure^k(G_x) = +\infty\}) = 0.
  \end{align*}
  We therefore have
  \begin{align*}
    \hausmeasure^{d - k}(X_k)
    = \hausmeasure^{d - k}(X_k \setminus X_{k,\epsilon})
    + \hausmeasure^{d - k}(X_{k, \epsilon} \cap X_k)
    = \hausmeasure^{d - k}(X_k \setminus X_{k,\epsilon})
    + \hausmeasure^{d - k}(X_{k, \epsilon}).
  \end{align*}
  To demonstrate $\hausmeasure^{d-k}(X_k) = \lim_{\epsilon \downarrow 0}
  \hausmeasure^{d-k}(X_{k,\epsilon}) > 0$, observe that if $\hausmeasure^{d
    - k}(X_k) < \infty$, continuity of measure from above
  gives the result.
  Otherwise $\cup_{\epsilon \in \Q_{> 0}} X_{k, \epsilon} = \{x \in X \mid
  \hausmeasure^k(G_x) > 0\} \supset X_k$, and continuity of measure from
  below applies.
  Lastly, let $\gamma > 0$.
  Then $\hausmeasure^{d - k +
    \gamma}(X_{k,\epsilon}) = 0$ because $\hausmeasure^{d -
    k}(X_{k,\epsilon}) < \infty$, and
  $X_k \setminus X_{k,\epsilon} = \{x \in X \mid 0 < \hausmeasure^k(G_x)
  < \epsilon\}$ satisfies
  \begin{align*}
    \int^*_{X_k \setminus X_{k,\epsilon}}
    \hausmeasure^{k - \gamma}(G_x)
    d\hausmeasure^{d - k + \gamma}(x)
    \le C_{k-\gamma, d - k + \gamma} \hausmeasure^d(G) < \infty.
  \end{align*}
  Because $\hausmeasure^{k - \gamma}(G_x) = +\infty$ for
  $x \in X_k \setminus X_{k,\epsilon}$,
  we have $\hausmeasure^{d - k + \gamma}(X_k \setminus X_{k,\epsilon}) = 0$.
\end{proof}

\subsection{Step~\ref{item:fiber-survival}: the empirical
  slices survive under sampling}
\label{sec:proof-fiber-survival}

\renewcommand{\gset}{X_0}
\renewcommand{\empgset}{X_{0,n}}

We specialize to the randomized sampling setting.
Let $x^0$ be the minimum norm solution
to $0 \in A_P(x) + \normalcone_X(x)$,
so $x^0$ also solves $0 \in A_P(x) + \normalcone_{X \cap C}(x)$ for
any closed convex set $C$ containing $x^0$.
Without loss of generality, we may thus take $X = X \cap B'$ for a ball $B'
\supset B = \{x^0 + b \ball\}$, where $b$ is the value in the statement of
Theorem~\ref{theorem:lebesgue-measure-positive} (i.e.,
Assumption~\ref{assumption:basic-vi-setting}.\ref{item:vi-moments}); we
slightly abuse notation to write $X$ for the remainder of the proof.
We also tacitly and without loss of generality assume $\lambda\epsilon \le
b$ (otherwise replace $\epsilon$ with $\epsilon' = b / \lambda$; the
argument applies, as this can only shrink
$\statpoints_P(\epsilon)$).

Lemma~\ref{lemma:minty-to-full-subset} shows that $K = x^0 + \lambda
\epsilon \ball$ satisfies $\res{\lambda}{A}(K) \subset
\statpoints_P(\epsilon) \cap K$, $\res{\lambda}{A}(K) \subset K$ is compact,
and that for $\epsilon \le \boundconst / 2 < \infty$, the piece $\gr (A_P +
\normalcone_X) \cap (\res{\lambda}{A}(K) \times (\boundconst/2) \ball)$ is
compact with Hausdorff and inductive dimension $d$, as is
\begin{align*}  
  G_P \defeq \gr (A_P + \normalcone_X) \cap \left(\res{\lambda}{A}(K)
  \times \boundconst \ball\right).
\end{align*}
Consider the empirical analogue
\begin{align}
  \label{eqn:empirical-graph}
  A_n \defeq A_{P_n} + \normalcone_X
  ~~ \mbox{and} ~~
  G_n \defeq \gr A_n \cap (\res{\lambda}{A}(K) \times \boundconst \ball)
\end{align}
of $G_P$, where we let $G_n(x) = A_n(x) \cap \boundconst \ball$ (we use
$\epsilon \le \boundconst/2$ to be able to apply
Lemma~\ref{lemma:convergence-after-truncation}).
Lemma~\ref{lemma:splitting-hausdorff-slices}
guarantees the existence of $0 \le k \le d$,
$\epsilon > 0$, and a compact
\begin{align*}
  \gset \defeq \left\{x \in \res{\lambda}{A}(K) \mid
  G_P(x)
  ~ \mbox{contains~a~}k\mbox{-simplex~} \simplex_k
  ~ \mbox{with} ~
  \epsilon \le \hausmeasure^k(\simplex_k) \right\}
\end{align*}
with $0 < \hausmeasure^{d - k}(X_0) < \infty$.
We then define its empirical analogue
\begin{align}
  \label{eqn:construct-empirical-goodset}
  \empgset \defeq \left\{x \in \gset \mid G_n(x)
  ~ \mbox{contains a} ~ k\mbox{-simplex~}\simplex_k \mbox{~with~}
  \hausmeasure^k(\Delta_k) \ge \frac{\epsilon}{2}\right\}.
\end{align}
Then with high probability, the set $\empgset$ has large $d-k$-dimensional
Hausdorff measure:
\begin{lemma}
  \label{lemma:empgset-is-large-whp}
  For all $\gamma > 0$,
  the random subset $\empgset$
  satisfies
  \begin{equation*}
    \P\left(\hausmeasure^{d-k}(\empgset)
    \geq (1 - \gamma) \cdot \hausmeasure^{d-k}(\gset)\right) \to 1.
  \end{equation*}
\end{lemma}
\begin{proof}
  For $x\in \gset$, the fiber $G_P(x) = A(x) \cap \boundconst \ball$
  contains a $k$-simplex $\simplex = \conv\{u_i\}_{i=0}^k$ with
  $k$-dimensional Hausdorff measure $\hausmeasure^k(\simplex) \ge \epsilon$
  whose vertices $u_i$ satisfy $\ltwo{u_i}\leq \boundconst$.
  The volume of the simplex is continuous with respect to its vertices,
  so it is uniformly continuous for vertices in the ball $\boundconst\ball$.
  Accordingly, there is
  $\delta > 0$ such that if
  $\wt{u}_i$ satisfy $\ltwos{\wt{u}_i - u_i} \le \delta$,
  the perturbed simplex $\wt{\simplex} = \conv\{\wt{u}_i\}_{i = 0}^{k}$
  has $k$-dimensional Hausdorff measure
  $\hausmeasure^k(\wt{\simplex}) \ge \epsilon / 2$.
  Therefore, for any $x \in \gset$,
  if $\dhaus(G_P(x), G_{n}(x))\leq \delta$,
  then $G_{n}(x)$ contains a $k$-simplex with $k$-dimensional
  Hausdorff measure at least $\frac{\epsilon}{2}$, as it is also convex.
  Equivalently,
  \begin{equation*}
    \empgset \supset \{x\in \gset \mid \dhaus(G_n(x), G_P(x))\leq \delta\}.
  \end{equation*}
  Note that $\empgset$ is compact by the argument for
  Observation~\ref{obs:Kk-compact}.

  Define the probability measure
  $\mu(B)=\frac{\hausmeasure^{d-k}(B\cap\gset)}{\hausmeasure^{d-k}(\gset)}$
  on $\gset$.
  Lemma~\ref{lemma:empirical-subdifferential-pointwise-normal} shows that
  $\E[\dhaus^2(A_P(x) + \normalcone_X(x), A_{P_n}(x) + \normalcone_X(x))]
  \to 0$ under Assumption~\ref{assumption:basic-vi-setting}, so the
  pointwise convergence in Lemma~\ref{lemma:convergence-after-truncation}
  applied to the truncated images $G_P(x)$ and $G_{n}(x)$ (recall that
  $\boundconst/2 \ge \epsilon$, so $G_P(x) \cap (r/2) \ball \neq \emptyset$)
  yields
  \begin{equation*}
    \E\left[
      \int_{\gset} \min\left\{1, \dhaus(G_P(x), G_{n}(x))\right\}d\mu(x)
      \right]
    \to 0,
  \end{equation*}
  by Fubini's theorem and dominated convergence.
  Therefore
  \begin{align*}
    \E[\mu(\gset\setminus \empgset)]
    &\leq
    \E[\mu(\{x \in X_0 ~ \mbox{s.t.} ~ \dhaus(G_P(x),G_{n}(x))>\delta \})]\\
    &\leq
    \frac{1}{\delta}
    \E\left[
      \int_{\gset} \min\left\{1, \dhaus(G_P(x), G_{n}(x))\right\}d\mu(x)
      \right]
    \to 0,
  \end{align*}
  so $\E[\mu(\empgset)]\to 1$ and
  the lemma follows.
\end{proof}

\subsection{Step~\ref{item:large-hausdorff}: the measure
  of graph pieces}
\label{sec:proof-large-hausdorff}

We complete the proof of
Theorem~\ref{theorem:lebesgue-measure-positive},
continuing in the setting of the previous section.
We first demonstrate that the empirical graph
piece~\eqref{eqn:empirical-graph}, $G_n = \gr(A_{P_n} + \normalcone_X) \cap
(\res{\lambda}{A}(K) \times \boundconst \ball)$,
has reasonable $d$-dimensional Hausdorff measure
for $K = x^0 + \lambda \epsilon \ball$.
Recall the projection mapping
$\pi(x, v) = x$, which is $1$-Lipschitz and whose
inverse relative to $G_n$,
$\pi|_{G_n}^{-1}(x) = \{x\} \times G_n(x)$, is bounded.
Then as in
the proof of
Lemma~\ref{lemma:splitting-hausdorff-slices},
\citet[Thm.~2.10.25]{Federer69}, 
implies
\begin{align*}
  \int_{X_{0,n}}^* \hausmeasure^k\left(G_n \cap \pi|_{G_n}^{-1}(x)\right)
  d \hausmeasure^{d - k}(x)
  \le C_{k,d-k} \cdot \hausmeasure^d(G_n).
\end{align*}
By the construction~\eqref{eqn:construct-empirical-goodset}
of $X_{0,n}$,
$\hausmeasure^k(G_n \cap \pi_{G_n}^{-1}(x)) \ge \epsilon/2$ for
each $x \in X_{0,n}$, and so
\begin{align*}
  \frac{\epsilon}{2} \hausmeasure^{d-k}(X_{0,n})
  \le C_{k,d-k} \hausmeasure^d(G_n).
\end{align*}

We may now combine the pieces: Lemma~\ref{lemma:empgset-is-large-whp}
shows that as $n \to \infty$, for any $\gamma > 0$ we have
$\hausmeasure^{d - k}(X_{0,n}) \ge (1 - \gamma) \hausmeasure^{d - k}(X_0)$
with high probability,
and the application of
Lemma~\ref{lemma:splitting-hausdorff-slices}
in step~\ref{item:fiber-survival}
guarantees $0  <\hausmeasure^{d - k}(X_0) < \infty$.
Summarizing, we have shown that
there exists a constant $c > 0$ such that
\begin{align}
  \label{eqn:hausdorff-of-empirical-graph}
  \P\left(\hausmeasure^d(G_n) \ge c \right) \to 1,
\end{align}
where $G_n$ is the empirical graph piece~\eqref{eqn:empirical-graph},
which for $K = \{x^0 + \lambda \epsilon \ball\}$ satisfies
\begin{align*}
  G_n \subset \gr A_n \cap
  \left(\statpoints_P(\epsilon) \cap K \times \boundconst \ball\right).
\end{align*}

As the final step, we use the Minty
parameterization~\eqref{eqn:minty-param}: we observe that
\begin{align*}
  \minty{\lambda}{A_n}^{-1}(G_n)
  = \{x + \lambda v \mid (x, v) \in G_n\}
  \subset \left\{x \mid \res{\lambda}{A_n}(x) \in \statpoints_P(\epsilon)
  \cap K \right\}
  \cap \{\res{\lambda}{A}(K) + \lambda \boundconst \ball\}
\end{align*}
by the construction~\eqref{eqn:empirical-graph}
as $(x, v) \in G_n$ implies $\ltwo{v} \le \boundconst$.
Lemma~\ref{lemma:minty-to-full-subset}
part~\ref{item:resolvent-near-stationary} shows that
$\res{\lambda}{A}(K) \subset K = x^0 + \lambda \epsilon \ball$, and so
we obtain
\begin{align*}
  \minty{\lambda}{A_n}^{-1}(G_n)
  \subset \res{\lambda}{A_n}^{-1}(\statpoints_P(\epsilon) \cap K)
  \cap \{x^0 + 2 \lambda \boundconst \ball\}.
\end{align*}
Thus
\begin{align*}
  \hausmeasure^d\left(\minty{\lambda}{A_n}^{-1}(G_n)\right)
  = \lebesgue\left(\minty{\lambda}{A_n}^{-1}(G_n)\right)
  \le \lebesgue\left(\res{\lambda}{A_n}^{-1}(\statpoints_P(\epsilon) \cap K)
  \cap \{x^0 + 2 \lambda \boundconst \ball\}
  \right)
\end{align*}
as $\hausmeasure^d = \lebesgue$ on $\R^d$.
Because the Minty parameterization is a bi-Lipschitz homeomorphism,
\citet[Prop.~3.5]{Maggi12} then shows that
\begin{align*}
  \hausmeasure^d(G_n) \le \left(1 + \lambda^{-2}\right)^{d/2}
  \hausmeasure^d\left(\minty{\lambda}{A_n}^{-1}(G_n)\right).
\end{align*}
The choice $A_n = A_{P_n} + \normalcone_{X \cap B'}$ in the proof of
step~\ref{item:fiber-survival} in Section~\ref{sec:proof-fiber-survival}
means that inequality~\eqref{eqn:hausdorff-of-empirical-graph} implies the
theorem, because $\boundconst \ge 2 \epsilon$ is otherwise arbitrary.

\subsection{Discussion of step~\ref{item:full-dim-parameterization}
  and graph decomposition}
\label{sec:discuss-dim-parameters}

The decomposition of $\gr \partial f$ into pieces $K_i = \{(x, g) \mid g \in
\partial f(x), \dim \partial f(x) = i\}$, which sampling approximately
preserves, constitutes the central tool in our results, and it guarantees
$\proxmap_{\poploss_{P_n}}^{-1}(\statpoints_P(\epsilon))$ has positive
Lebesgue measure for $\epsilon > 0$.
%
%
The lack of quantitative guarantees
makes it challenging to use as an ingredient to obtain rates of convergence.
Part of this is that the ``bases'' $\pi(K_i) = \{x \mid \dim \partial f(x) =
i\}$ take many forms, as they can be countable, non-convex, and disconnected,
making the application of convex-analytic tools non-obvious.
We thus provide examples illustrating the
decomposition in Lemma~\ref{lemma:splitting-hausdorff-slices},
after which we discuss the challenges
of making the result more quantitative.

\begin{example}[The differentiable case]
  For differentiable convex $f$
  and $G = \gr \nabla f$,
  we have $G_x = \{\nabla f(x)\}$.
  For compact and full-dimensional $X \subset \R^d$,
  we have $K_0 = X$ and $K_i = \{x \mid \dim \partial f(x) = i\}
  = \emptyset$ for each $i \ge 1$,
  because $\dim \partial f(x) = 0$ for each $x$.
\end{example}

\begin{example}[A simple non-differentiable case]
  \label{example:l1-parts}
  For $f(x) = |x|$ and $X = [-1, 1]$, the
  graph $\gr \partial f$ has inductive and Hausdorff dimension 1.
  The decomposition in Lemma~\ref{lemma:splitting-hausdorff-slices} becomes
  \begin{align*}
    K_0 = \left\{(x, \sign(x)) \mid 0 < |x| \le 1\right\}
    ~~ \mbox{and} ~~
    K_1 = \left\{(0, g) \mid |g| \le 1\right\},
  \end{align*}
  where $\hausmeasure^1(K_0) = 2$ and
  $\hausmeasure^1(K_1) = 2$.
  In part~\ref{item:special-base-graph} of the lemma, we may take $X_0 =
  \openright{-1}{0} \cup \openleft{0}{1}$, as $\hausmeasure^0(\{1\}) =
  \hausmeasure^0(\{-1\}) = 1$, or $X_1 = \{0\}$, as $\hausmeasure^1([-1,1])
  = 2$, and $G_0 = [-1,1] = \conv\{-1, 1\}$ is a $1$-simplex.
  See Figure~\ref{fig:example-graph-figs}(a) for an illustration.
\end{example}

\begin{example}[Disconnected base sets]
  \label{example:disconnected-base}
  Let $\{q_i\}_{i \in \N}$ index the rationals
  $0 < q < 1$
  and define the non-decreasing step function
  $g(x) = \sum_{i : q_i \le x} \frac{1}{2^i}$,
  which satisfies $g(0) = 0$  and $g(1) = 1$.
  Then $f(x) = \int_0^x g(t) dt$ is convex
  on $[0, 1]$, and
  \begin{align*}
    \partial f(x) & = \left[\sum_{i : q_i < x} \frac{1}{2^i},
      \sum_{i : q_i \le x} \frac{1}{2^i}\right]
    ~~ \mbox{for~} x \in X_1 \defeq \Q \cap [0, 1] \\
    \partial f(x) & = \{\nabla f(x)\} = \left\{\sum_{i : q_i < x} \frac{1}{2^i}
    \right\}
    ~~ \mbox{for~} x \in X_0 \defeq [0, 1] \setminus \Q.
  \end{align*}
  This partitions the graph $\gr \partial f \subset [0, 1] \times [0, 1]$ into
  the two sets
  \begin{align*}
    K_0 = \bigcup_{x \in [0, 1] \setminus \Q}
    \{x\} \times \{\nabla f(x)\}
    ~~~ \mbox{and} ~~~
    K_1 = \bigcup_{x \in \Q \cap [0, 1]}
    \{x\} \times \partial f(x).
  \end{align*}
  For the projection $\pi(x, v) = x$, the set $X_0 = \pi(K_0) = [0, 1]
  \setminus \Q$ is 1-dimensional with $\dimhaus \partial f(x) = \inddim
  \partial f(x) = 0$ for each $x \in X_0$, while $\partial f(x)$ is a
  $1$-dimensional interval for $x \in X_1 = \pi(K_1) = \Q \cap [0,1]$.
  Both $X_0$ and $X_1$ are disconnected.
\end{example}

\begin{example}[Curved base sets]
  \label{example:curved-base}
  For $f_0(x) = \ltwo{x - e_1}^2$ and $f_1(x) = \half \ltwo{x + e_1}^2$,
  define $f(x) = \max\{f_0(x), f_1(x)\}$,
  which 
  $x^0 = (3 - \sqrt{8}) e_1$ minimizes.
  Then $f_0 = f_1$ on
  \begin{align*}
    S \defeq \{x \mid \ltwo{x}^2 = 6
    x_1 - 1\},
  \end{align*}
  and the
  subdifferential $\partial f(x) = \conv\{2 (x - e_1), (x
  + e_1)\}$ is $1$-dimensional
  on $S$.
  Otherwise $f$ is differentiable with $\nabla f(x) = \nabla f_i(x)$ if
  $f_i(x) > f_{1 - i}(x)$, $i = 0, 1$.
  This partitions $\gr \partial f$ into
  \begin{align*}
    K_0 = \left\{(x, \nabla f(x)) \mid x \not\in S
    \right\}
    ~~ \mbox{and} ~~
    K_1 = \left\{\{x\} \times \conv\{2(x - e_1),
    (x + e_1)\} \mid x \in S\right\}
  \end{align*}
  and yields base sets $X_1 = S$ and
  $X_0 = S^c$, with dimensions $d - 1$ and $d$.
  We have
  $\hausmeasure^1(\partial f(x)) \gtrsim 1$ for $x \in X_1$,
  while $\hausmeasure^0(\partial f(x)) = 1$ for $x \in X_0$.
  Figure~\ref{fig:example-graph-figs}(b) illustrates this example.
\end{example}

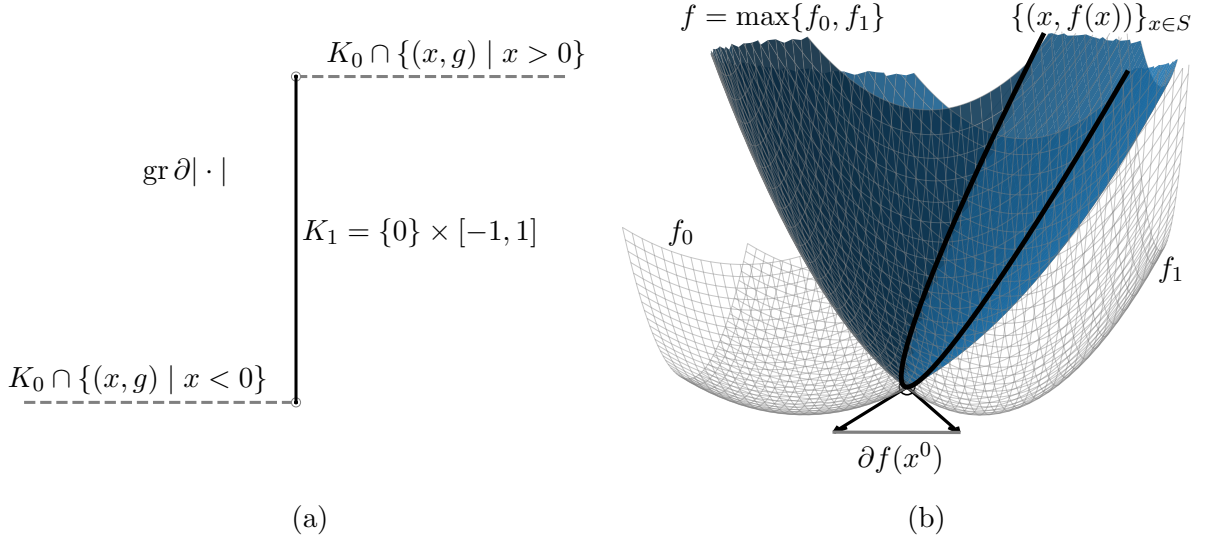
\begin{figure}[ht]
  \begin{center}
    \begin{tabular}{cc}
      \hspace{-.5cm}
      \begin{overpic}[width=.52\columnwidth]{
          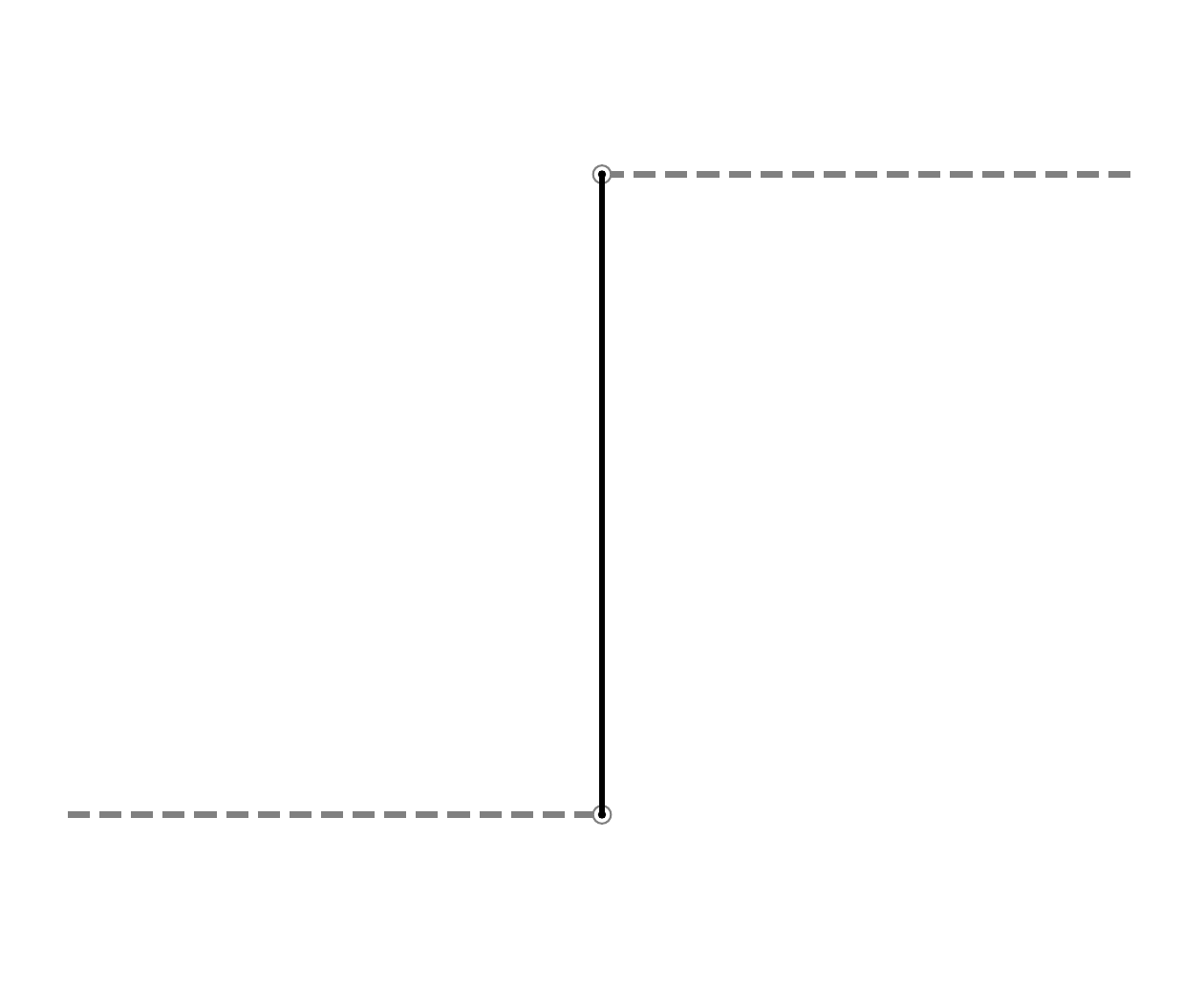}
        \put(25,51){$\gr \partial |\cdot|$}
        \put(51,41){$K_1 = \{0\} \times [-1, 1]$}
        \put(3,17){$K_0 \cap \{(x, g) \mid x < 0\}$}
        \put(55,70){$K_0 \cap \{(x, g) \mid x > 0\}$}
      \end{overpic}
      &
      \hspace{-.5cm}
      \begin{overpic}[width=.52\columnwidth]{
          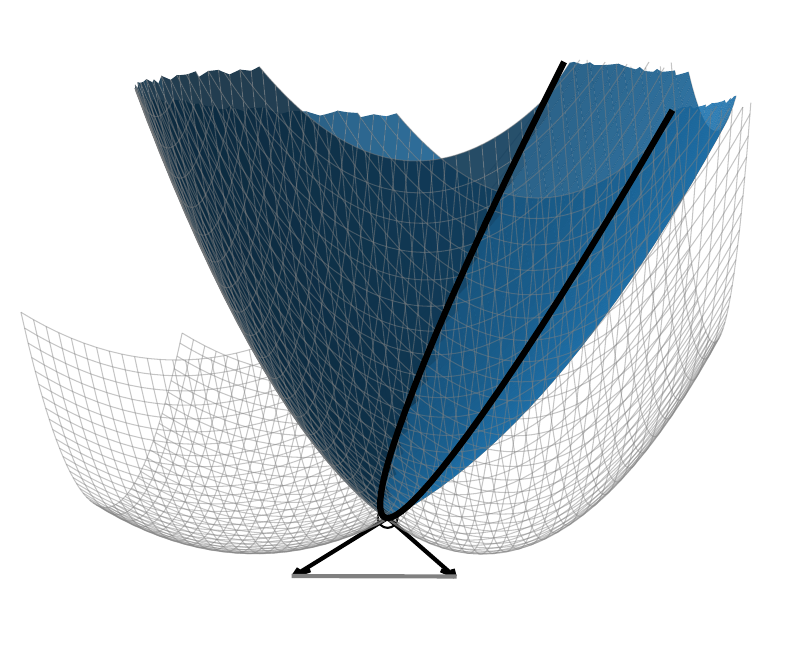}
        \put(10,41){$f_0$}
        \put(90,35){$f_1$}
        \put(12,76){$f = \max\{f_0, f_1\}$}
        \put(41,4){$\partial f(x^0)$}
        \put(66,76){$\{(x, f(x))\}_{x \in S}$}
      \end{overpic} \\
      (a) & (b)
    \end{tabular}
    \caption{\label{fig:example-graph-figs}
      Illustrations of partitions of $\gr \partial f$.
      (a) The graph $\gr \partial f$ of $f(x) = |x|$ as in
      Example~\ref{example:l1-parts}.
      The set $K_1 = \{0\} \times [-1,1] = \{0\} \times \partial f(0)$
      corresponds to the vertical segment, while $K_0 = \R_{> 0} \times
      \{1\} \cup \R_{< 0} \times \{-1\}$ consists of the two partially open horizontal
      segments.
      (b) The function $f(x) = \max\{f_0(x), f_1(x)\}$ for
      $f_0(x) = \ltwo{x - e_1}^2$ and $f_1(x) = \half \ltwo{x + e_1}^2$
      in Example~\ref{example:curved-base}.
      On the illustrated $1$-dimensional
      sphere $S = \{x \mid \ltwo{x}^2 = 6x_1 - 1\}$
      (dark black line) where $f_0 = f_1$,
      the subdifferentials have $\dim \partial f(x) = 1$,
      as the subdifferential at $x^0 = \argmin_x f(x)$ shows. 
    }
  \end{center}
\end{figure}

With these examples, we discuss
points~\ref{item:base-graph-set}
and~\ref{item:special-base-graph} in
Lemma~\ref{lemma:splitting-hausdorff-slices},
which provide a non-quantitative guarantee that
$\hausmeasure^{d-k}(X_k) > 0$ and that the
images $\partial f(x)$ for $x \in X_k$ have positive
$k$-dimensional Hausdorff measure $\hausmeasure^k(\partial f(x))$.
%
The challenging part of this result is to demonstrate that the base set
$X_k = \pi(K)$ for the projection $\pi(x, v) = x$ has positive Hausdorff
measure $\hausmeasure^{d-k}(\pi(K)) > 0$.
To do so, we use inductive dimension (leveraging the dimension of mappings),
to (essentially) derive two-sided bounds on the Fubini-type
integral~\eqref{eqn:projected-K-finite} relating $\hausmeasure^d(K_i)$ and
$\hausmeasure^{d - i}(\pi(K_i))$ for the partitioning sets $K_i$ in
part~\ref{item:partition-by-hausdorff} of the lemma.
%

Alberti, Ambrosio, Cannarsa, and Soner~\citep{AlbertiAmCa92, AmbrosioCaSo93,
  Alberti94, AlbertiAm99} study similar structural properties of the
dimensionality of singular sets $\Sigma^k(f) \defeq \{x \in \R^d \mid \dim
\partial f(x) \ge k\}$ of (weakly) convex functions $f$.
They show~\citep{AlbertiAmCa92, Alberti94} that $\Sigma^k(f)$ is countably
$(d - k)$ rectifiable, implying that $\dimhaus \Sigma^k(f) \le d - k$,
corresponding to the upper bounds $\hausmeasure^{d-k}(X_{k,\epsilon}) <
\infty$ in part~\ref{item:base-graph-set} of
Lemma~\ref{lemma:splitting-hausdorff-slices}.
The lower bounds are more challenging; \citet[Thm.~2.3]{AmbrosioCaSo93} show
that in the neighborhood of a point $x$ with $\dim \partial f(x) = k$, there
is an $m = m(f, x)$ with $0 \le m \le k$ for which
$\hausmeasure^{d-k}(\Sigma^m(f) \cap \{x + \rho \ball\}) \ge (1 + o(1))
\omega_{d-k} \rho^{d-k}$ as $\rho \downarrow 0$.
They~\cite[Sec.~1]{AmbrosioCaSo93} define $m$ to be the
Carath\'{e}odory number of the set $\Gamma_x \defeq \limsup_{y \to x} \nabla
f(y)$ minus 1, noting that the set of limiting gradients satisfies
$\conv(\Gamma_x) = \partial
f(x)$~\cite[e.g.][Thm.~VI.6.3.1]{HiriartUrrutyLe93}.
Unfortunately, there exist convex $f$ with unique minimizers $x^0$ for which
$m = m(f, x^0) = 0$ while $d - k$ is
arbitrary~\cite[Ex.~2.1]{AmbrosioCaSo93}, and so we cannot generally
leverage these results in Lemma~\ref{lemma:splitting-hausdorff-slices}.


\providecommand{\jacobian}{\mathsf{J}}

The co-area formula~\cite[Thm.~3.2.22]{Federer69} analogizes Fubini's theorem
for Hausdorff measures, and states that for sufficiently regular sets $K
\subset \R^N$ and $u : K \to X \subset \R^m$ with $m
\le n$,
\begin{equation}
  \label{eqn:co-area}
  \int_{K} \jacobian_m^K u(y) d\hausmeasure^n(y)
  = \int_{X} \hausmeasure^{n-m}\left(u^{-1}(x)\right)
  d\hausmeasure^m(x),
\end{equation}
where $\jacobian_m^K u$ denotes the $m$-dimensional Jacobian of $u$ along the
tangents to $K$.
In the integral~\eqref{eqn:projected-K-finite}, we have $u = \pi$, the
(orthogonal) projection onto the first $d$ components,
$N = 2d$, $n = d$, and $m = d - k$.
The projection satisfies $0 \le \jacobian_{d-k}^K \pi \le 1$, and
so (with regularity)
\begin{align*}
  \int_K \jacobian_{d-k}^K \pi(y) d\hausmeasure^d(y)
  = \int_{\pi(K)} \hausmeasure^{k}(G \cap \pi^{-1}(x)) d\hausmeasure^{d-k}(x).
\end{align*}
One might hope to guarantee a lower bound $\jacobian_{d-k}^K \pi \ge c >
0$ to develop a two-sided inequality
\begin{align*}
  c \cdot \hausmeasure^d(K)
  \le \int_{\pi(K)} \hausmeasure^k(G \cap \pi^{-1}(x))
  d\hausmeasure^{d-k}(x) \le
  \hausmeasure^d(K),
\end{align*}
which would allow quantitative bounds on $\hausmeasure^{d-k}(\pi(K))$ in
(e.g.) Lemma~\ref{lemma:empgset-is-large-whp}.

Even for $1$-Lipschitz $\mc{C}^\infty$ convex functions, such
a quantitative lower bound cannot generally hold.
For example, for $\epsilon > 0$, consider $f_\epsilon(x) = \epsilon
\log(e^{x/\epsilon} + e^{-x/\epsilon})$,
which is $1$-Lipschitz with $f_\epsilon'(x) =
(e^{2x/\epsilon} - 1) / (e^{2x/\epsilon} + 1)$ and second derivative
$f_\epsilon''(0) = \epsilon^{-1}$.
Because the $\mc{C}^\infty$ map $x \mapsto (x, f_\epsilon'(x))$
parametrizes the $1$-dimensional set $K_0 = \{(x,
f_\epsilon'(x))\}_{-1 \le x \le 1}$, this has $1$-dimensional Jacobian
$\jacobian_1^{K_0} \pi(0) = 1 / \sqrt{1 + f''_\epsilon(0)^2} = (1 + o(1))
\cdot \epsilon$ as $\epsilon \to 0$.
This somewhat trivial example shows that, at least following our
proof approach, new insights may be necessary to provide
quantitative guarantees on the characteristics of the set
$\proxmap_{\poploss_{P_n}}^{-1}(\statpoints(\epsilon))$.


%


\section{Discussion}

Under local moment and maximal monotonicity assumptions, the randomized
resampling estimator $\what{x}_n$ that Algorithm~\ref{alg:resample} defines
satisfies
\begin{equation*}
  \dist\left(0, \popop(\what{x}_n) + \normalcone_X(\what{x}_n)
  \right) \to 0
  ~~ \mbox{in probability}.
\end{equation*}
For the subdifferential $\partial \poploss_P$, this gives the convergence
$\dist(0, \partial \poploss_P(\what{x}_n) + \normalcone_X(\what{x}_n)) \cp
0$.
Our guarantee does not require the operator to be single-valued or Lipschitz
continuous, and does not rely on a smoothed surrogate as a proxy for
stationarity.
This distinction is central to our contribution: we control the stationary
residual $\dist\left(0, \popop(\what{x}_n) +
\normalcone_X(\what{x}_n)\right)$ directly.
Many statistical and optimization questions remain open.

A natural strengthening of our results would show that the random resampling
step in Algorithm~\ref{alg:resample} is unnecessary,
and that the empirical proximal mapping maps a neighborhood of $x^0$
onto the $\epsilon$-stationary set.
No such neighborhood exists in general: an extension of the construction in
Section~\ref{sec:baby-challenges}, where we use $\loss_z(x) = \sup_j z_j
\hinge{\<x, u_j\> - \alpha_j}$ for specially chosen $u_j \in \sphere^1$ and
$\alpha_j > 0$ that we provide in Appendix~\ref{sec:proof-prox-challenges},
shows this must fail.
\begin{proposition}
  \label{proposition:prox-challenges}
  Let $0 < \delta < 1$ and $Z \in \{0, 1\}^\N$ have i.i.d.\ $\bernoulli(1 -
  \delta)$ coordinates.
  The point $x^0 = 0$ minimizes each of the $1$-Lipschitz
  convex functions $\loss_z : \R^2 \to \R_+$.
  Define the event
  $\mc{B}_n$ that for all neighborhoods $O$ of $x^0$,
  there is an open set $U \subset O$ such that
  \begin{enumerate}[leftmargin=*,label=(\roman*)]
  \item $\proxmap_{\lambda \poploss_{P_n}}(x) = x$ for each
    $x \in U$ and $\lambda > 0$, and
  \item 
    $\poploss_P$ is differentiable on $U$,
    and $\norm{\nabla \poploss_P(x)} \ge 1 - \delta$ for $x \in U$.
  \end{enumerate}
  Then $\P(\mc{B}_n) = 1$ for all $n$.
\end{proposition}



Written differently, for $\epsilon < 1 - \delta$ and stationary set
$\statpoints_P(\epsilon) = \{x \mid \dist(0, \partial
\poploss_P(x)) \le \epsilon\}$,
\begin{align*}
  \P\left(\mbox{there~exists~a~neighborhood~}U~\mbox{of~}x^0
  ~ \mbox{and}~\lambda > 0
  ~ \mbox{s.t.}~
  \proxmap_{\lambda \poploss_{P_n}}(U) \subset \statpoints_P(\epsilon)
  \right) = 0
\end{align*}
for all sample sizes $n$.
That all neighborhoods of the minimum norm minimizer $x^0$ of $\poploss_P$
contain points for which the empirical proximal mapping $\proxmap_{\lambda
  \poploss_{P_n}}$ keeps points quite non-stationary, justifies some of the
complexity in our analyses.
For example, it is not obvious how to correct an empirical regularized
estimator to be (approximately) stationary.

Nonetheless, we might hope that the empirical proximal mapping
$\proxmap_{\poploss_{P_n}}$ maps ``most'' points into the stationary set.
We record this as our first open question:
\begin{question}
  Let Assumption~\ref{assumption:basic-vi-setting} hold.
  For each $\epsilon, \lambda > 0$, does there exist a
  neighborhood $U$ of $x^0$ such that $\lebesgue(U \setminus
  \resolvent_{\lambda A_{P_n} + \normalcone_X}^{-1}
  (\statpoints_P(\epsilon))) \to 0$ with high probability?
\end{question}

Special cases allow simpler procedures: when the population objective
$\poploss_P$ is differentiable at the minimum norm minimizer $x^0$ and when
$\partial \poploss_P(x^0)$ contains a neighborhood of $0$ (i.e.,
$\poploss_P$ has a sharp minimizer).
In each of these, regularized minimizers~\eqref{eqn:l2-empirical-vi} exhibit
asymptotic stationarity.
For simplicity, we consider cases when
Assumption~\ref{assumption:basic-vi-setting}.\ref{item:vi-moments}
holds with $p$ moments near $x^0$, meaning that $\loss_z$ is
$\lipconst(z)$-Lipschitz at $x^0$ with $\E[\lipconst^p(Z)] < \infty$.

\begin{observation}
  Assume that
  $0 \in \interior (\partial \poploss_P(x^0) + \normalcone_X(x^0))$
  and that $\partial \loss_z$ has $p > 2$ moments near $x^0$.
  Then for any sequence $\regmult_n \downarrow 0$,
  the regularized minimizer~\eqref{eqn:l2-empirical-vi} satisfies
  \begin{equation*}
    \P\left(\what{x}^{\regmult_n} = x^0
    ~ \mbox{eventually}\right) = 1.
  \end{equation*}
\end{observation}
\begin{proof}
  For convex $h$, if $0 \in \interior \partial h(x^0)$, there exists
  $\epsilon > 0$ such that
  $\interior \partial h(x^0) \supset \epsilon \ball$, and
  $\ltwo{g} \ge \epsilon$ for all $g \in \partial h(x)$ and
  $x \neq x^0$.
  Indeed, monotonicity of subgradients implies that
  $v = (x - x^0) / \ltwos{x - x^0}$ satisfies
  $\inf\{\<g, v\> \mid g \in \partial h(x)\} \ge h'(x^0; v) \ge \epsilon$.
  With this guarantee, take $h = \poploss_P + \convexindic{X}$ and assume
  $\partial \poploss_P(x^0) + \normalcone_X(x^0) \supset c \ball$ for some
  $c > 0$.
  Then Lemma~\ref{lemma:empirical-subdifferential-pointwise} and an
  application of the Borel-Cantelli lemma imply that with probability $1$,
  eventually $\dhaus(\partial \poploss_P(x^0), \partial \poploss_{P_n}(x^0))
  \le \frac{c}{2}$.
  On this event, we have
  $\partial \poploss_{P_n}(x^0) + \normalcone_X(x^0) \supset \frac{c}{2} \ball$,
  so once $\regmult_n$ is small enough that
  $\regmult_n \ltwos{x^0} < \frac{c}{2}$, we obtain
  \begin{equation*}
    0 \in \interior \left(\partial \Big(\poploss_{P_n}(x)
    + \frac{\regmult_n}{2} \ltwo{x}^2\Big)\Big|_{x = x^0}
    + \normalcone_X(x^0)\right).
  \end{equation*}
  When this occurs, $x^0$ uniquely minimizes
  $\poploss_{P_n}(x) + \frac{\lambda}{2} \ltwo{x}^2$ over $x \in X$.
\end{proof}

\begin{observation}
  Assume that $\poploss_P$ is differentiable at $x^0 \in \interior X$, and
  that $\partial \loss_z$ has $p \ge 2$-moments near $x^0$.
  Then if $\regmult_n \to 0$ while $\regmult_n \gg 1/\sqrt{n}$,
  \begin{equation*}
    \dist\left(0, \partial \poploss_P(\what{x}^{\regmult_n})
    + \normalcone_X(\what{x}^{\regmult_n})
    \right) \cp 0.
  \end{equation*}
\end{observation}
\begin{proof}
  By Proposition~\ref{proposition:sampled-vi-convergence},
  $\what{x}^{\regmult_n} \cp x^0$.
  The outer semi-continuity~\cite[Thm.~6.2.4]{HiriartUrrutyLe93} of
  the subdifferentials of a convex function then implies that
  $x \mapsto \dist(0, \partial \poploss_P(x) + \normalcone_X(x))$ is
  continuous at $x^0$, and the continuous mapping theorem gives the
  result.
\end{proof}

One additional question regards uniformity of convergence: our results are
pointwise in the distribution $P$.
This contrasts with classical optimization.
To make the contrast apparent, suppose that the functions $\{\loss_z\}_{z
  \in \mc{Z}}$ are $\lipconst$-Lipschitz and the domain $X$ is bounded.
Then stochastic gradient methods achieve (worst-case optimal) convergence
rate $\E[\poploss_P(x_k)] - \poploss_P\opt \lesssim \lipconst \cdot
\diam(X) / \sqrt{k}$, independent of $P$ and
dimension~\cite{NemirovskiJuLaSh09, AgarwalBaRaWa12}, and letting $\mc{P}$
be the collection of distributions on $\mc{Z}$, the distribution-free
guarantee $\inf_{\what{x}_n} \sup_{P \in \mc{P}}
\E_{P^n}[\poploss_P(\what{x}_n) - \poploss_P\opt] \to 0$ holds if and only
if the losses $\{\loss_z\}_{z \in \mc{Z}}$ are uniformly Lipschitz on each
compact $X_0 \subset \relint X$~\cite{ArecesDu25}.
In a companion paper we are preparing, we show that in distinction,
procedures achieving stationarity~\eqref{eqn:stationary-point-goal} must
exhibit some dependence on the underlying dimension.


We nonetheless have a natural open question:
\begin{question}
  \label{question:uniform}
  Let $d < \infty$ and $X \subset \R^d$ be compact convex.
  Assume that $\{\loss_z\}_{z \in \mc{Z}}$ are $\lipconst$-Lipschitz
  and convex on $\R^d$.
  Is there an estimator sequence $\what{x}_n = \what{x}_n(Z_1, \ldots, Z_n)$
  such that
  \begin{align*}
    \lim_n
    \sup_{P \in \mc{P}}
    \E_{P^n}\left[
      \dist\left(0, \partial \poploss_P(\what{x}_n) + \normalcone_X(\what{x}_n)
      \right)\right]
    = 0?
  \end{align*}
\end{question}
\noindent
That is, we ask whether it is possible to derive a convergence rate to
stationarity for $\lipconst$-Lipschitz convex objectives.
%
%
One avenue would be to adapt \citeauthor{ShalevShSrSr09}'s stability-based
arguments~\citep{ShalevShSrSr09}, which show that stochastic optimization is
possible if and only if estimators exist that are appropriately stable to
replacing an individual example $Z_i$.
%
Thus far, the natural instability of subdifferential mappings has made it
difficult for us to generalize replace-one stability to Hausdorff or other
set distances, leaving Question~\ref{question:uniform} quite open.

\appendix

\section{Miscellaneous deferred proofs}

\subsection{Proof of Observation~\ref{observation:conditional-coverage}}
\label{sec:proof-conditional-coverage}

Writing the subdifferential
of $\loss_{x,z}(\theta) = \alpha \hinge{\theta^\top \phi(x) - z}
+ (1 - \alpha) \hinge{z - \theta^\top \phi(x)}$
and using the shorthand $\tau_\theta(x) = \theta^\top \phi(x)$,
we obtain
\begin{align*}
  \lefteqn{\partial \poploss_P(\theta)} \\
  & = \E_P\big[\alpha \phi(X) \indic{Z < \tau_\theta(X)}
    - (1 - \alpha) \phi(X) \indic{Z > \tau_\theta(X)}
    + [-(1 - \alpha), \alpha] \cdot \phi(X) \indic{Z = \tau_\theta(X)}
    \big] \\
  & = -\E\left[\phi(X) \indic{Z > \tau_\theta(X)}\right]
  + \alpha \E[\phi(X)]
  + \E\left[[-1, 0] \cdot \phi(X) \indic{Z = \theta^\top \phi(X)}\right].
\end{align*}
So if there is a vector $u \in \partial \poploss_P(\theta)$,
there is a functional $0 \le s(x, z) \le 1$ for which
\begin{align*}
  u = \alpha \E[\phi(X)] - \E\left[\phi(X) \indic{Z > \tau_\theta(X)}\right]
  - \E\left[s(X, Z) \cdot \phi(X) \indic{Z = \tau_\theta(X)}\right].
\end{align*}
Rewriting this by conditioning  on $X$,
we have
\begin{align*}
  u = \E\left[\phi(X) \cdot (\alpha - P(Z > \tau_\theta(X) \mid X))\right]
  - \E\left[s(X, Z) \cdot \phi(X) \indic{Z = \tau_\theta(X)}\right].
\end{align*}
In particular, for all vectors $v \in \R^d$ and nonnegative
weighting functions of the form $w(x) = v^\top \phi(x)$,
we have
\begin{align*}
  u^\top v
  \le \E\left[w(X) \cdot \left(\alpha - P(Z > \tau_\theta(X) \mid X)\right)
    \right]
\end{align*}
or, renormalizing,
$P_w(Z > \tau_\theta(X)) \le \alpha + \epsilon \ltwo{v} / \E[w(X)]$.

\subsection{Constructing a separated subset of the sphere}
\label{sec:construct-non-alignment}

For $\epsilon > 0$ and $x\in \sphere^{d-1}$,
let $C_\epsilon = \{y \in \sphere^{d-1} \mid \ltwo{y - x} \le \epsilon\}$ 
be the spherical cap centered at $x$.
Then for $\epsilon$ small
the surface area $\textup{SA}_d(C_\epsilon) \le c_d \epsilon^{d-1}$ while
$\textup{SA}_d(\sphere^{d-1}) = \frac{2 \pi^{d/2}}{\Gamma(d/2)} = C_d$,
where $c_d$ and $C_d$ are dimension dependent constants.
Let $\varepsilon_i > 0$ be a non-increasing sequence
satisfying $\frac{c_d}{C_d} \sum_{i = 1}^\infty  \varepsilon_i^{d-1} \le 1$,
and recursively define a set of points as follows.
Take $U_1 = \sphere^{d-1}$, $u_1 \in U_1$, and construct
\begin{align*}
  U_{i + 1} = U_i \setminus \{u_i + \varepsilon_i \ball_2^d\},
\end{align*}
taking $u_i \in U_i$.
At each step $i$, any point $u_j$ for $j < i$ satisfies
$\ltwo{u_i - u_j} \ge \varepsilon_j$,
while each point $j > i$ satisfies $\ltwo{u_i - u_j} \ge \varepsilon_i$;
we can never have $U_i = \emptyset$ by a volume argument
because $\frac{c_d}{C_d} \sum_i \varepsilon_i^{d-1} \le 1$.
So $\ltwo{u_i - u_j} \ge \varepsilon_i$ for each $i \neq j$,
and thus we have inequality~\eqref{eqn:non-alignment}:
\begin{align*}
  \sup_{j \neq i} \<u_i, u_j\>
  = \sup_{j \neq i} \left\{1 - \half \ltwo{u_i - u_j}^2 \right\}
  \le 1 - \half \varepsilon_i^2 < 1.
\end{align*}

\subsection{Proof of Lemma~\ref{lemma:minimum-norm-solution}}
\label{sec:proof-minimum-norm-solution}

First, we claim that $\ltwos{x^0} \ge \ltwos{x^\lambda}$ for all $\lambda
> 0$.
By definition of $x^0$ and $x^\lambda$ as solutions
to their respective V.I.s, we have
\begin{align*}
  0 \in A(x^0) + \normalcone_X(x^0)
  ~~ \mbox{and} ~~
  -\lambda x^\lambda \in A(x^\lambda) + \normalcone_X(x^\lambda).
\end{align*}
Then by monotonicity, for the elements
$a^0 \in A(x^0)$ and $w^0 \in \normalcone_X(x^0)$
and $a^\lambda \in A(x^\lambda)$ and $w^\lambda \in \normalcone_X(x^\lambda)$
realizing $0 = a^0 + w^0$ and $a^\lambda + w^\lambda + \lambda x^\lambda = 0$,
\begin{align*}
  \<0 +\lambda x^\lambda, x^0
  - x^\lambda\>
  = \<a^0 + w^0 - a^\lambda - w^\lambda, x^0 - x^\lambda\> \ge 0.
\end{align*}
Rearranging and applying Cauchy-Schwarz,
\begin{align*}
  \lambda \ltwos{x^\lambda}^2
  & \le \lambda \<x^0, x^\lambda\>
  \le \lambda \ltwos{x^0} \ltwos{x^\lambda},
\end{align*}
and dividing by $\ltwos{x^\lambda}$ gives the claim (except if
$\ltwos{x^\lambda} = 0$, in which case it is trivial).

We now argue that $x^\lambda \to x^0$.
Consider any potential limit of a (subsequence)
of $x^\lambda$ as $\lambda \downarrow 0$, i.e.,
$x^\lambda \to x$.
We claim $x$ solves~\eqref{eqn:vi}.
The outer semicontinuity of $A+\normalcone_X$ yields
$\limsup_{\lambda \to 0} A(x^\lambda)+\normalcone_X(x^\lambda) 
\subset A(x)+\normalcone_X(x)$;
let $u^\lambda \in A(x^\lambda)+\normalcone_X(x^\lambda)$ 
be the vector satisfying $u^\lambda+\lambda x^\lambda=0$.
Then as $\lambda x^\lambda \to 0$, we obtain
that $\lim_{\lambda \to 0} u^\lambda = 0 
\in A(x)+\normalcone_X(x)$.
Hence, since $\ltwos{x^0} \ge \ltwos{x^\lambda}$, $x$ 
must be the minimum norm element of $X\opt$.

\subsection{Proof of Proposition~\ref{proposition:sampled-vi-convergence}}
\label{sec:proof-sampled-vi-convergence}

We use a typical stochastic convergence argument.
\begin{lemma}
  Let $\what{x}^{\regmult} \in X$ solve the empirical
  problem~\eqref{eqn:l2-empirical-vi}.
  Then there exists an element
  $a_z(x) \in A_z(x)$ with $a_P(x) = \E_P[a_Z(x)]$ for which
  \begin{align*}
    \ltwobig{\what{x}^{\regmult} - x^\regmult} \le \frac{1}{\regmult}
    \ltwo{a_{P_n}(x^\regmult)
      - a_P(x^\regmult)}.
  \end{align*}
\end{lemma}
\begin{proof}
  By definition of $x^\regmult$, there
  exists $w^\regmult \in \normalcone_X(x^\regmult)$ and
  a selection $a_z(x) \in A_z(x)$ for which
  $a_P(x^\regmult)
  = \int a_z(x^\regmult) dP(z)$
  satisfies $0 = a_P(x^\regmult) + \regmult x^\regmult + w^\regmult$.
  Let $\wb{a}_n \in A_{P_n}(\what{x})$ and
  $w_n \in \normalcone_X(\what{x})$ be the
  empirical element of $A_{P_n}(\what{x}^{\regmult_n})$ and
  $\normalcone_X(\what{x})$ satisfying
  $0 = \wb{a}_n + w_n + \regmult \what{x}$,
  and let $\wb{a}^\regmult_n
  = \frac{1}{n} \sum_{i = 1}^n a_{Z_i}(x^\regmult)$.
  Then by monotonicity
  \begin{align*}
    \<\wb{a}_n - \wb{a}_n^\regmult + w_n - w^\regmult
    + \regmult (\what{x} - x^\regmult), \what{x} - x^\regmult\>
    & = \<\wb{a}_n - \wb{a}_n^\regmult + w_n - w^\regmult,
    \what{x} - x^\regmult\> + \regmult \ltwobig{\what{x} - x^\regmult}^2 \\
    & \ge \regmult \ltwobig{\what{x} - x^\regmult}^2.
  \end{align*}
  Using that $0 = \wb{a}_n + w_n + \regmult \what{x}$ implies
  \begin{align*}
    -\<\wb{a}_n^\regmult + w^\regmult + \regmult x^\regmult, \what{x} - x^\regmult\>
    \ge \regmult \ltwobig{\what{x} - x^\regmult}^2.
  \end{align*}
  Adding
  $\<a_P(x^\regmult)+\regmult x^\regmult+w^\regmult, \what{x} - x^\regmult\>=0$
  to the left and applying Cauchy-Schwarz yields
  \begin{align*}
    \ltwo{a_P(x^\regmult) - \wb{a}_n^\regmult}
    \ltwobig{\what{x} - x^\regmult}
    \ge \regmult \ltwobig{\what{x} - x^\regmult}^2,
  \end{align*}
  which is the desired result.
\end{proof}

The final step uses Assumption~\ref{assumption:basic-vi-setting}.
We see that for small enough $\regmult$, $x^\regmult$ belongs to the
neighborhood of $x^0$ Assumption~\ref{assumption:basic-vi-setting}
specifies, and so
\begin{align*}
  \ltwobig{\what{x}^{\regmult} - x^0}
  \le \ltwobig{x^\regmult - x^0}
  + \frac{1}{\regmult} \ltwobig{a_{P_n}(x^\regmult) - a_P(x^\regmult)}.
\end{align*}
Taking an expectation yields that $\E[\ltwos{a_{P_n}(x^\regmult) -
    a_P(x^\regmult)}^2] \lesssim \frac{1}{n}$, and so long as $n \regmult_n^2
\to \infty$ while $\regmult_n \to 0$ we have the first result.
For the almost-sure convergence, observe by the vector Khinchine-Kahane
inequalities and a symmetrization argument~\cite[Ch.~1.3]{delaPenaGi99} that
for any mean zero vectors $W_i \in \R^d$, we have for $p > 2$ that
\begin{align*}
  \E\left[\ltwobigg{\sum_{i = 1}^n W_i}^p\right]
  \le 2^{p} \E\left[\ltwobigg{\sum_{i = 1}^n \varepsilon_i W_i}^p\right]
  & \lesssim
  \E\left[\E\left[\ltwobigg{\sum_{i = 1}^n \varepsilon_i W_i}^2 \mid W_1^n
      \right]^{p/2}\right] \\
  & = \E\left[\bigg(\sum_{i = 1}^n \ltwo{W_i}^2\bigg)^{p/2}\right]
  \le n^{p/2 - 1} \sum_{i = 1}^n \E[\ltwo{W_i}^p],
\end{align*}
where $\varepsilon \simiid \uniform\{-1, 1\}$ are an independent Rademacher
sequence.
Accordingly,
given $p$th moment, we obtain
\begin{align*}
  \E\left[\ltwobig{a_{P_n}(x^\regmult) - a_P(x^\regmult)}^p\right]
  \lesssim n^{-p/2} \E\left[\ltwobig{a_Z(x^\regmult)}^p\right],
\end{align*}
so that
\begin{align*}
  \sum_{n = 1}^\infty \P\left(\frac{1}{\regmult_n} \ltwobig{a_{P_n}(x^{\regmult_n})
    - a_P(x^{\regmult_n})} > \epsilon\right)
  \lesssim \frac{1}{\epsilon^p} \sum_{n = 1}^\infty
  \frac{1}{\regmult_n^p \cdot n^{p/2}}
  \cdot \E\left[\ltwobig{a_Z(x^\regmult)}^p\right].
\end{align*}
So long as $\sum_n \frac{1}{\regmult_n^p n^{p/2}} < \infty$,
the Borel-Cantelli lemma implies almost sure convergence.

\subsection{Proof of Proposition~\ref{proposition:prox-challenges}}
\label{sec:proof-prox-challenges}

As in Section~\ref{sec:baby-challenges}, we construct piecewise linear
functions, but now construct them so that their points of
non-differentiability approach $x^0 = 0$ along a curve.
From this, we demonstrate that there is always an open set $U$
arbitrarily close to $x^0$ for which the proximal mapping
$\proxmap_{\poploss_{P_n}}$ is the identity on $U$, but over which
$\poploss_P$ has large gradient.

Let $0 < t_j \le 1, j \in \N$ be a strictly decreasing positive sequence
with $t_j \downarrow 0$.
For $z \in \{0, 1\}^\N$, define the unit vectors and losses
\begin{align}
  \label{eqn:loss-counterexample}
  u_j = \left[\begin{matrix} 1 \\ t_j \end{matrix}\right] /
  \sqrt{1 + t_j^2}
  ~~~ \mbox{and} ~~~
  \loss_z(x) \defeq \sup_j z_j \hinge{\<x, u_j\> - \alpha_j}
\end{align}
for scalars $\alpha_j = t_j^2 / (1 + t_j^2)^{1/2} > 0$,
which we choose for convenience in the following
lemma.
\begin{lemma}
  \label{lemma:loss-counterexample}
  Let the loss $\loss_z$ have the structure~\eqref{eqn:loss-counterexample}.
  Then there exists a sequence of points $x_j \to 0$ and an open
  neighborhood $U_j$ of $x_j$ for which
  \begin{equation*}
    \loss_z(x) =
    \begin{cases} 0 & \mbox{if}~z_j = 0 \\
      \<x, u_j\> - \alpha_j > 0
      & \mbox{if}~ z_j = 1
    \end{cases}
    ~~~ \mbox{for~all~} x \in U_j.
  \end{equation*}
\end{lemma}
\begin{proof}
  The points $v_j = (-t_j^2, 2 t_j)$ satisfy $v_j \to 0$ and
  $\<v_j, u_i\>
  = \frac{1}{(1 + t_i^2)^{1/2}}
  \cdot (2 t_i t_j - t_j^2)$,
  so
  \begin{align*}
    \<v_j, u_i\> - \alpha_i
    = \<v_j, u_i\> - \frac{t_i^2}{(1 + t_i^2)^{1/2}}
    = -\frac{1}{(1 + t_i^2)^{1/2}} (t_i - t_j)^2.
  \end{align*}
  In particular, $\<v_i, u_i\>-\alpha_i = 0$, while
  \begin{equation*}
    \sup_{j \neq i} (\<v_j, u_i\>-\alpha_i) 
    = -(1 + t_i^2)^{-1/2} \min\{(t_i - t_{i+1})^2,(t_i - t_{i-1})^2\}.
  \end{equation*}
  We construct $x_i \to 0$ via perturbations of the points $v_i$.
  Set $x_j= v_j + \eta_j u_j$ for the scalar $\eta_j = \min_{i \neq j}
  \frac{1}{2 (1 + t_i^2)^{1/2}} (t_i - t_j)^2 > 0$, which satisfies
  $\norm{x_j} \lesssim t_j \to 0$; observe that if $i \neq j$,
  \begin{align*}
    \<x_j, u_i\> - \alpha_i
    = -\frac{1}{(1 + t_i^2)^{1/2}} (t_i - t_j)^2
    + \eta_j \<u_i, u_j\>
    \le -\frac{1}{2 (1 + t_i^2)^{1/2}} (t_i - t_j)^2
    \le -\eta_j.
  \end{align*}
  On the other hand, for $i = j$,
  \begin{align*}
    \<x_i, u_i\> - \alpha_i = \eta_i \<u_i, u_i\>
    = \eta_i > 0.
  \end{align*}
  Take the open sets $U_j$ to be the open $\ell_2$-balls
  of radius $\eta_j > 0$ around $x_j$.
\end{proof}

By the lemma, we see that if $z_j = 1$, then
$\nabla \loss_z(x) = u_j$ for $x \in U_j$, while
$\nabla \loss_z(x) = 0$ for $x \in U_j$ if $z_j = 0$.
So $\nabla \poploss_P(x) = (1 - \delta) u_j$ for $x \in U_j$.
On the other hand, by the second Borel-Cantelli lemma, because $\P(P_n(Z_j =
0) = 1) = \delta^n > 0$,
\begin{align*}
  \P\left(P_n(Z_j = 0) = 1 ~ \mbox{for~countably~many}~ j \in \N\right) = 1.
\end{align*}
On the probability 1 event $\mc{E}$ that for countably many $j$, the empirical
probability $P_n(Z_j = 0) = 1$ (i.e., coordinate $j$ is always 0),
for any finite $J$ there is $j > J$ such that $P_n(Z_j = 0) = 1$.
On the event $\mc{E}$, because $\diam(U_j) \to 0$, for all $b > 0$ there
exists an (arbitrarily large) index $j$ such that $U_j \subset x^0 + b
\ball$ while $P_n(Z_j = 0) = 1$.
Then $\nabla \poploss_{P_n}(x) = 0$ for each $x \in U_j$,
while $\nabla \poploss_P(x) = (1 - \delta) u_j$.
For each $\lambda > 0$,
$\proxmap_{\lambda \poploss_{P_n}}(x) = x$ for all $x \in U_j$.

\section{Deferred proofs for
  Theorem~\ref{theorem:lebesgue-measure-positive}}

\subsection{Proof of Lemma~\ref{lemma:empirical-subdifferential-pointwise}}
\label{sec:proof-empirical-subdifferential-pointwise}

Using the standard identification~\cite[Thm.~V.3.3.8]{HiriartUrrutyLe93} of
the Hausdorff distance with support functions of sets, that is, that
$\dhaus(A, B) = \sup_{\ltwo{v} \le 1} |\sigma_A(v) - \sigma_B(v)|$ for
compact convex $A, B$ with $\sigma_A(v) = \sup_{a \in A} \<a, v\>$, we see
that
\begin{align*}
  \dhaus\left(S_{P_n}, S_P\right)
  & = \sup_{\ltwo{v} \le 1}
    \left|\frac{1}{n} \sum_{i = 1}^n
    \sigma_{S_{Z_i}}(v) - \sigma_{S_P}(v)
    \right|.
\end{align*}
Introducing the random signs
$\varepsilon_i \simiid \uniform\{-1, 1\}$, a symmetrization
argument~\cite[e.g.][]{delaPenaGi99} gives
\begin{align*}
  \E\left[\dhaus^p\left(S_{P_n}, S_P\right)\right]
  & \le
  C_p \cdot \E\left[\sup_{\ltwo{v} \le 1}
  \left|\frac{1}{n} \sum_{i = 1}^n \varepsilon_i \sigma_{S_{Z_i}}(v)
    \right|^p\right],
\end{align*}
where $C_p$ depends only on $p$.

Define the function $H_n(\varepsilon)
\defeq \sup_{v \in \ball}
|\sum_{i = 1}^n \varepsilon_i \sigma_{S_{Z_i}}(v)|$, which is
a $\sqrt{\sum_{i = 1}^n \lipconst^2(Z_i)}$-Lipschitz function
of $\varepsilon \in \{-1, 1\}^n$.
Then the convex concentration inequality~\cite[e.g.][Thm.~3.24]{Wainwright19}
states that
\begin{align*}
  \P\left(
  \left|H_n(\varepsilon)
  - \E[H_n(\varepsilon) \mid P_n]\right| \ge t \mid P_n
  \right) \le 2 \exp\left(-\frac{t^2}{2 n \E_{P_n}[\lipconst^2(Z)]}
  \right)
\end{align*}
for all $t \ge 0$, where $\E_{P_n}[\lipconst^2(Z)] = \frac{1}{n} \sum_{i =
  1}^n \lipconst^2(Z_i)$.
Using that $\E[X] = \int_0^\infty  P(X \ge t) dt$ for $X \ge 0$,
we thus obtain
\begin{align*}
  \E\left[n^p \dhaus^p\left(S_{P_n}, S_P\right)\right]
  & \le C_p \cdot \E\left[|H_n(\varepsilon)|^p\right] \\
  & \le 2^p C_p \cdot \left(\E\left[
    \E\left[
    \left|H_n(\varepsilon) - \E[H_n(\varepsilon)
      \mid P_n]\right|^p \mid P_n\right]\right]
  + \E\left[\left|\E[H_n(\varepsilon) \mid P_n]\right|^p
    \right]\right).
\end{align*}
For the first inner conditional expectation,
we observe that
\begin{align*}
  \E\left[\left|H_n(\varepsilon) - \E[H_n(\varepsilon)
      \mid P_n]\right|^p \mid P_n\right]
  & = \int_0^\infty \P\left(|H_n(\varepsilon) - \E[H_n(\varepsilon) \mid P_n]
  | \ge t^{1/p}\right) dt \\
  & \le 2 \int_0^\infty \exp\left(-\frac{t^{2/p}}{2n \E_{P_n}[\lipconst^2(Z)]}
  \right) dt \\
  & 
  = p \Gamma(p) n^{p/2} \E_{P_n}[\lipconst^2(Z)]^{p/2}
\end{align*}
by a standard gamma integral computation.
To upper bound the expectation
$\E[|H_n(\varepsilon)| \mid P_n]$ we use a standard
entropy integral~\cite[e.g.]{Wainwright19, VanDerVaartWe96}
to obtain
$\E[|H_n(\varepsilon)| \mid P_n] \lesssim
\sqrt{d n \E_{P_n}[\lipconst^2(Z)]}$.
Substituting above yields
\begin{align*}
  \E\left[n^p \dhaus^p(S_{P_n}, S_P)\right]
  & \le C_p \cdot
  (dn)^{p/2} \cdot \E\left[\E_{P_n}[\lipconst^2(Z)]^{p/2}\right]
  \le C_p \cdot 
  (dn)^{p/2} \cdot \E[\lipconst^p(Z)]
\end{align*}
by Jensen's inequality.

\subsection{Proof of Lemma~\ref{lemma:truncating-hausdorff}}
\label{sec:proof-truncating-hausdorff}

Let $\delta = \dhaus(C, D) \le \frac{c}{2}$.
Then each $x \in D$ satisfies $\dist(x, C) \le \delta$,
and similarly, each $y \in C$ satisfies $\dist(y, D) \le \delta$.
We proceed in three steps: first,
arguing that $C \cap 2B$ is non-empty, then
arguing that $\dist(x, C \cap 2B) \lesssim \delta$ for all
$x \in D \cap 2B$, and
finally that $\dist(y, D \cap 2B) \lesssim \delta$
for all $y \in C \cap 2B$.

\begin{enumerate}[leftmargin=*,label=\arabic*.]
\item
  We first argue that $C \cap 2B$ is non-empty.
  Take $x_0 \in D \cap B$ and the $y_0 \in C$ satisfying $\ltwo{x_0 -
    y_0} \le \delta$, so $\ltwo{y_0} \le \ltwo{x_0} + \delta < 2c$,
  and $y_0 \in 2 B$.
  Thus $C \cap 2B$ is non-empty.
\item
  Take $x \in D \cap 2B$, and let $y \in C$ satisfy $\ltwo{x - y} \le
  \delta$,
  so $\ltwo{y} \le 2c + \delta$.
  Let $0 < t < 1$ to be chosen
  and let
  \begin{align*}
    z = (1 - t) y + t y_0 \in C,
  \end{align*}
  where the inclusion follows by convexity of $C$.
  Computing norms, we have
  \begin{align*}
    \ltwo{z} & \le (1 - t) \ltwo{y} + t \ltwo{y_0}
    \le (1 - t) (\ltwo{x} + \delta) + t (\ltwo{x_0} + \delta) \\
    & \le (1 - t) (2c + \delta)
    + t(c + \delta)
    = 2 c - t c + \delta,
  \end{align*}
  which satisfies $\ltwo{z} \le 2c$ whenever
  $\delta \le t c$, that is,
  $z \in 2B$.
  Taking $t = \frac{\delta}{c}$ then yields
  \begin{align*}
    \ltwo{z - x} \le
    \ltwo{z - y} + \ltwo{x - y}
    \le t \ltwo{y - y_0} + \delta
    \le \frac{\delta}{c}(c + \delta + 2c + \delta)
    + \delta
    \le 6 \delta.
  \end{align*}
\item Take $y \in C \cap 2B$.
  Let $x \in D$ be the closest point to $y$, so 
  $\ltwo{x - y} \le \delta$
  and $\ltwo{x} \le 2c + \delta$.
  Take any point $x' \in D \cap B$, and for $0 \le t \le 1$ to
  be chosen let
  $z = (1 - t) x + t x' \in D$.
  Evidently, if $(1 - t) (2c + \delta) + t c
  = 2c - tc + (1 - t) \delta \le 2 c$,
  i.e., $\delta \le t (c + \delta)$,
  then $z \in D \cap 2B$; we may take
  $t = \delta / c$ to achieve this
  because $c \ge 2\delta$.
  In this case
  \begin{align*}
    \ltwo{y - z}
    \le \ltwo{y - x} + t \ltwo{x - x'}
    \le \delta + \frac{\delta}{c}
    (2c + \delta + c)
    \le 5 \delta.
  \end{align*}
\end{enumerate}
As $\delta = \dhaus(C, D) \le c/2$, we see that
\begin{align*}
  \dhaus(C, D) \le \frac{c}{2}
  ~~ \mbox{implies} ~~
  \dhaus(C \cap 2B, D \cap 2B)
  \le 6 \cdot \dhaus(C, D).
\end{align*}

\bibliographystyle{abbrvnat}
\bibliography{bib}

\end{document}